 \documentclass[final,5p,times,twocolumn,authoryear]{elsarticle}

\usepackage{amssymb}
\usepackage{lipsum}

\usepackage{algorithm}
\usepackage{algorithmic}

\usepackage{caption}
\usepackage{subcaption}
\usepackage[utf8]{inputenc} 
\usepackage[T1]{fontenc}    
\usepackage{hyperref}       
\usepackage{url}            
\usepackage{booktabs}       
\usepackage{amsfonts}       
\usepackage{nicefrac}       
\usepackage{microtype}      
\usepackage{xcolor}         
\usepackage{amsmath}
\usepackage{graphicx}  
\usepackage{float}  
\usepackage{algorithm}
\usepackage{algorithmic}
\usepackage{multirow}
\usepackage{enumerate}
\usepackage{amsthm}
\usepackage{colortbl}

\newtheorem{theorem}{Theorem}
\newtheorem{definition}{Definition}

\newtheorem{remark}{Remark}

\usepackage{wrapfig}
\usepackage{multicol}
\usepackage{times}
\usepackage{soul}
\usepackage{url}
\usepackage{subcaption}

\journal{Neural Network}

\begin{document}

\begin{frontmatter}

\title{Contrastive Desensitization Learning for Cross Domain Face Forgery Detection}

\renewcommand{\thefootnote}{\fnsymbol{footnote}} 

\author[first]{Lingyu Qiu\footnotemark[1]$^,$}
\author[first]{Ke Jiang\footnotemark[1]$^,$}
\author[first]{Xiaoyang Tan}
\address[first]{College of Computer Science and Technology, Nanjing University of Aeronautics and Astronautics \& MIIT Key Laboratory of Pattern Analysis and Machine Intelligence, China}
\begin{abstract}
     In this paper, we propose a new cross-domain face forgery detection method that is insensitive to different and possibly unseen forgery methods while ensuring an acceptable low false positive rate. Although existing face forgery detection methods are applicable to multiple domains to some degree, they often come with a high false positive rate, which can greatly disrupt the usability of the system. To address this issue, we propose an Contrastive Desensitization Network (CDN) based on a robust desensitization algorithm, which captures the essential domain characteristics through learning them from domain transformation over pairs of genuine face images. One advantage of CDN lies in that the learnt face representation is theoretical justified with regard to the its robustness against the domain changes. Extensive experiments over large-scale benchmark datasets demonstrate that our method achieves a much lower false alarm rate with improved detection accuracy compared to several state-of-the-art methods.
\end{abstract}




\end{frontmatter}

\setcounter{secnumdepth}{0} 

\definecolor{darkgreen}{RGB}{0, 128, 0}

\renewcommand{\thefootnote}{\fnsymbol{footnote}} 
\footnotetext[1]{Equal contribution}

\section{Introduction}


The application of face recognition has recently achieved great success with the development of deep learning techniques. However, existing face recognition systems are vulnerable when facing face forgery attacks, where it is possible to generate fake faces through complex manipulation of face images. Therefore, it is essential to develop anti-face forgery methods, aiming to distinguish real face images from those manipulated by forgery techniques. These methods are also critical to defend against fake news, defame celebrities and break authentication, which can bring about serious damages to the political, social, and security areas~\cite{lyu2020deepfake}.

Current face manipulation methods can be roughly classified into four categories~\cite{tolosana2020deepfakes,peng2024local}: entire face synthesis, identity swapping, attribute manipulation, and expression swapping. The identity swapping, which is also known as Deepfakes~\cite{Synthesizing1,Synthesizing2,Synthesizing3}, is arguably one of the most harmful face forgery methods among them, and has attracted widespread attention~\cite{sun2021domain}. Traditional face forgery detection methods~\cite{DFDC_Frequency1,DFDC_Frequency2,twostream,ke2023df} train the detectors in a supervised way to capture the specific patterns in those manipulated images. However, no one knows the number of face forgery methods that will emerge in the future, making it crucial to explore how to achieve insensitivity to different and possibly unseen forgery methods. This highlights the importance of cross-domain face forgery detection, where 'domain' usually refers to different distribution that generates the face images of interest. As the forgery samples could be constructed in a heterogeneous manner, the mismatch between different domains (i.e., {\it domain shift}) is almost inevitable, and this may bring about great challenges to traditional detection methods in which rational decisions on the target domain can only be possible under the condition that enough training samples from the same domains are available.

\begin{figure}
    \centering
    \setlength{\abovecaptionskip}{0.cm}
    \setlength{\belowcaptionskip}{-0.5cm}
    \includegraphics[width=1\linewidth]{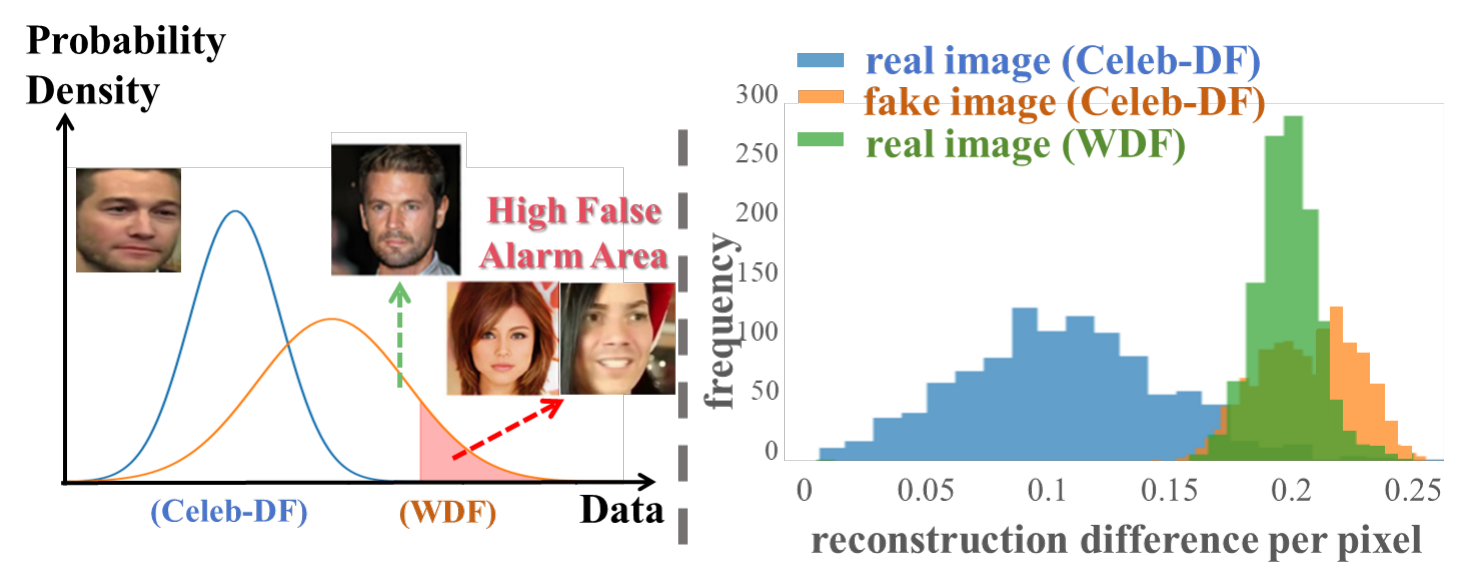}
    \caption{The diagram (left) of the {\it domain shift} problem, shows that the divergence between the source and target data distribution would potentially lead to a high false alarm rate. We also perform reconstruction (right) over cross-domain samples, and observe that the distribution of \textbf{real} face images reconstructed from the target dataset (WildDeepfake) differs significantly from those from the source domain (Celeb-DF) while having large overlapping with that of the fake face images of the same source domain (Celeb-DF).}
    \label{fig:introduct}
\end{figure}

In order to address this issue, a natural method is to treat the cross-domain face forgery detection problem as a domain adaptation problem ~\cite{chen2012marginalized}. For example, in ~\cite{RECCE,SBI,self-supervised1,shi2023real,super-resolution}, a two-stage strategy is adopted in which a generative model of real face images is first learned without using any fake images and then a face image is treated as fake only if it appears to be an outlier with regard to the learned manifold of real faces. As this representation learning stage does not involve any fake images, it is essentially insensitive to various forgery domains. 
However, an easily overlooked issue of this type of method is that they tend to misclassify genuine images undergone domain shift, leading to a significant false alarm region (see Figure \ref{fig:introduct} for an illustration, all the images depicted are real images). It is well-known that false alarms can greatly disrupt the overall system usability, making it crucial to investigate the problem of how to achieve insensitivity to different forgery methods while ensuring an acceptable low false positive rate. 

To this end, in this paper we proposed a novel deep learning method, termed \textbf{C}onstrastive \textbf{D}esensitization \textbf{N}etwork (CDN). The  primary goal of CDN is to learn a {\it general} representation for real faces across multiple domains, so as to facilitate a low false alarm rate for real ones while maintaining a high detection rate for forged face images. For this purpose, our key idea is to construct a desensitization network that effectively captures the intrinsic characteristics of real faces shared by multiple domains while removing those domain-dependent style features. We implement this by first mixing the low-level visual features from different domains, as they are known to be more shareable than the high-level semantic features~\cite{2020GMFAD}. The desensitization network is then rewarded for achieving high reconstruction fidelity using the learned representations, even when faced with such distortions.

\textcolor{black}{
In summary, the core contributions of this work are threefold:
\begin{enumerate}
\item We propose a novel \textbf{contrastive desensitization learning} framework to learn domain-invariant representations for cross-domain deepfake detection (DFDC) tasks, effectively addressing feature entanglement caused by domain shifts.
\item We establish a \textbf{theoretically rigorous framework} grounded in variational inference, which formally guarantees the disentanglement of domain-specific and intrinsic features.
\item Through comprehensive experiments on multiple benchmarks (e.g., FaceForensics++, Celeb-DF), we demonstrate state-of-the-art performance in cross-domain scenarios, accompanied by systematic ablation studies and visual interpretability analyses.
\end{enumerate}
}

\section{Related Work}
In this section, we briefly review some previous works that are closely related to the current work.
\subsection{Supervised Face Forgery Detection} 
The rapid spread of face forgery technology has brought about an urgent requirement to develop forgery detectors. Many early methods are based on the detection of specific forgery patterns such as local noise~\cite{twostream,liang2023depth}, texture and high-level semantic feature sets~\cite{MultiAtt,luo2021generalizing,he2024gazeforensics} and frequency artifacts~\cite{DFDC_Frequency1,DFDC_Frequency2,DFDC_Frequency3}, to distinguish fake faces from real faces. However, one of the major disadvantages of these methods is that they are effective only in some limited scenarios where forgery patterns can be easily obtained from training data and remain relatively stationary across different domains.

\subsection{Cross-domain Face Forgery Detection} 
To promote the generalization to future or unseen forgery methods, recently several authors have proposed to use domain adaptation techniques to bridge the gap between different forgery domains. For example, in~\citep{Xception}, depthwise separable convolution is introduced to enhance the ability to capture more generalizabile patterns for face forgery detection. In ~\cite{DomainAdaption_DANN_2021}, a Domain-Adversarial Neural Network is introduced to learn domain-invariant features. The domain gaps between different forgery domains can also be narrowed through augmented bridging samples, as in ~\cite{DomainAdapation_TMM2023}, while in ~\cite{Guo_2023_ICCV_DA} a guide-space based method is proposed to separate real and different forgery domains in a controllable manner. In ~\cite{sun2021domain}, a
learning-to-weight (LTW) method based on the meta-learning technique is proposed to enhance the face detection performance across multiple domains. In RECCE~\citep{RECCE}, an unsupervised task (i.e., reconstruction) is introduced to enhance the robustness in detecting the cross-domain fake images. 
However, RECCE solely relies on the reconstruction error as an auxiliary task, which can potentially result in misclassifying cross-domain real samples and subsequently increase the risk of false alarms.
Unlike the previous approaches, we focus on modeling the generative process of real faces in different domains based on their low-level visual features to enhance the detection performance while maintaining an acceptable low false positive rate(FPR). 
Unlike Mixup~\cite{zhou2023mixstyle,li2024takes}, which simply shuffles features within a batch, our method specifically mixes statistical features from real images and jointly constrains their representations through both the encoder and decoder. This approach offers enhanced generalizability due to its domain-agnostic properties.

\section{Cross Domain Desensitization Learning}\label{sec_desensitization}
\subsection{Problem Settings and Motivations}\label{sec_problem}
A face forgery detection problem could be formulated as a binary classification problem using a latent variable model. To be specific, the observed data $(x, y)$ are assumed to be sampled from a fixed but unknown joint distribution $P(X, Y)$. To model this generative process, we assume that there exists an encoder $\theta$ which encodes $x$ with a hidden variable $z$. A classifier $\eta$ is then used to make the prediction of $x$ based on its hidden feature $z$. Let $l$ be some loss function, then the learning objective can be defined as follows,
\begin{align}
    \min_{\theta, \eta} \mathbb E_{(x,y)\sim P(X,Y),z\sim P_\theta(z|x)} [l(\eta(z), y)]\label{pre:objective}
\end{align}

To generalize this formulation to the cross-domain setting, for a given set of domains, 
we assume that each domain follows a prior distribution $P(D)$, and is responsible for a data generative process $P(X,Y|D)$. Then the aforementioned data generation process can be decomposed over these domains, as $P(X,Y) = \sum_D P(D)P(X,Y|D)$, and our ultimate goal is to search for an optimal encoder-predictor pair $(\theta, \eta)$, such that the following statistical risk is minimized, 
\begin{align}
    \min_{\theta, \eta} \mathbb E_{d\sim P(D)}\mathbb E_{(x,y)\sim P(X,Y|d),z\sim P_\theta(z|x)} [l(\eta(z), y)]\label{eq_objective_crossdomain}
\end{align}
However, in practice, we may only be accessible to an empirical domain distribution $\hat{P}(D)$. The mismatch between $\hat{P}(D)$ and $P(D)$ can lead to the {\it domain shift} problem, especially when there exists significant unbalance among the numbers of samples observed in different domains. 

To deal with this problem, we assume that there exists some mapping $F$ that decomposes a given data point $x$ into two parts, i.e., $(I,D)=F(X)$, where $I$ denotes intrinsic features and $D$ the domain-specific information. Our goal is then to seek a domain-invariant representation $Z$ with the following conditional independent properties,

\begin{definition} (Domain-invariant representation)
We define a representation \( Z \) of a data point \( X \), sampled from \( P(X) \), as \textbf{domain-invariant} if it is conditionally independent of the domain-specific information \( D \), i.e., \( Z \perp \!\!\! \perp D | I \), where \( I \) represents the intrinsic features of \( X \).
    \label{def:domain-invariant}
\end{definition}

In other words, the representation in Definition \ref{def:domain-invariant} is independent of the domain-specific information, hence being invariant to domain changes. Essentially this requires removing the domain-specific information from the input sample - a procedure we call desensitization. For this we first perform feature decomposition, as described next.

\noindent\textcolor{black}{
\textbf{Discussion.} Here we discuss about the difference between the problem settings in our paper and the well-known \textit{Domain Adaption} problem~\cite{song2022adaptive,lv2024domainforensics}. The foundational premises of the \textit{Domain Adaption} paradigm requires access to target domain samples during training - essentially a \textit{few-shot generalization} framework. In stark contrast, our work pioneers a \textit{zero-shot domain generalization} framework that strictly prohibits any target domain exposure during training. This distinction elevates the problem complexity by orders of magnitude, as models in our framework must achieve robust generalization to completely unseen data distributions through intrinsic domain invariance learning, rather than relying on target domain fine-tuning.
}

\textcolor{black}{
In the following sections, we present a comprehensive theoretical framework for Cross-domain Desensitization Learning (CDL), specifically designed to address the challenging \textit{zero-shot domain generalization} problem. Our methodological exposition proceeds through three pivotal components: First, Theorem \ref{pro:domaininvariant} establishes a critical theoretical connection between our cross-domain desensitization objective (Eq. \ref{eq:denoisingreconstruction_DTV}) and the domain-invariant representation criterion formalized in Definition \ref{def:domain-invariant}. Building upon this theoretical foundation, we subsequently develop a novel denoising reconstruction mechanism (Eq. \ref{eq:denoisingreconstruction}) that operationalizes these theoretical insights. The practical validity of our approach is formally guaranteed by Theorem \ref{theorem:denoising}, which bridges the gap between theoretical formulation and practical implementation by proving that our customized loss function in Eq. (\ref{eq:denoisingreconstruction}) provides a computationally tractable surrogate for optimizing the theoretically-motivated KL-divergence objective in Eq. (\ref{eq:denoisingreconstruction_DTV}).
}

\subsection{Feature Decomposition} \label{subsec_features}
The first step of the proposed method is to decompose the input information, i.e., to find a mapping $F$ that decomposes a given data point $x$ into intrinsic feature $I$ and the domain-specific information $D$. For this we use an encoder $\theta$, that is, $F(\theta;X) = (D,I)$, as follows. First, let the output of the encoder $\theta$ for an input $x$ be $z$, which is assumed to be sampled from the distribution $P_\theta(z|x)$ under the Gaussian assumption. Then following ~\cite{2013Efficient}, domain-style information $D$ can be defined based on the statistics of $X$ in the hidden space, while $I$ is defined to be the domain-normalized features that contain all the information of $X$ except domain information. This leads to the following explicit expression of the decomposition $F$, 
\begin{align}
    D = (\mu(z), \sigma(z))\quad\quad\quad\quad I = \frac{z - \mu(z)}{\sigma(z)}\label{eq:intrinsic}
\end{align}
where $\mu$, $\sigma$ are respectively the mean and variance of $z$.

\subsection{Cross Domain Desensitization}\label{dbc_module}

To learn a representation as defined in Definition \ref{def:domain-invariant}, we first decompose a given real face $x_A$ into two parts, i.e., the intrinsic feature $i_A$ and  the domain information $d_A$, using the function $F$ obtained in the previous section. Then we learn the desired representation $z$ using the following contrasting objective for a pair of real face images,
\begin{align}
\min_{\theta}\mathbb E_{x_A,x_B\sim P(X|Y=0)} [D_{KL}\big(P_\theta(z|i_A, d_B)\big\|P_\theta(z|i_A, d_A)\big)]\label{eq:denoisingreconstruction_DTV}
\end{align}
where $D_{KL}$ is the KL-divergence, $P(X|Y=0)$ is the distribution of real face images (with label $Y=0$), and $P_\theta(Z|I, D)$ is the mechanism $\theta$ that generates the representation $Z$ based on the feature decomposition $I,D$ from any input $X$. In particular, the following result reveals that the optimal solution of Eq.(\ref{eq:denoisingreconstruction_DTV}) ensures a domain-invariant representation, thereby satisfying Definition \ref{def:domain-invariant}.
\begin{theorem}
\label{pro:domaininvariant}
    The optimal solution of minimizing Eq. (\ref{eq:denoisingreconstruction_DTV}) guarantees the representation is conditionally independent of the domain information, i.e., $Z\perp \!\!\! \perp D|I$.
\end{theorem}

\begin{proof}
    \textcolor{black}{
    We remark that the optimal solution $\theta^*$ of Eq. (\ref{eq:denoisingreconstruction_DTV}) is $P_{\theta^*}(z|i_A, d_B) = P_{\theta^*}(z|i_A, d_A)$, due to the fact that $P_\phi(i_A, d_A|z)>0, \forall i_A, d_A, z$. This leads to,
    \begin{align}
        &\forall d_A,d_B\sim P(D), P_\theta^*(z|i_A, d_B) = P_\theta^*(z|i_A, d_A) \label{eq:proof11}\\
        \Rightarrow^{(a)} & \forall d_A,d_B\sim P(D), P_\theta^*(z|i_A, d_A, d_B) = P_\theta^*(z|i_A, d_A)=P_\theta^*(z|i_A, d_B)\nonumber\\
        \Rightarrow^{(b)} &\forall d\sim P(D), P_\theta^*(z|i_A, d) = \mathbb E_{d'\sim P(d'|i_A)}P_\theta^*(z|i_A, d, d')\nonumber\\
        &\quad\quad\quad\quad\quad\quad\quad\quad=\mathbb E_{d'\sim P(d'|i_A)}P_\theta^*(z|i_A, d')\nonumber\\ &\quad\quad\quad\quad\quad\quad\quad\quad= P_\theta^*(z|i_A)\label{eq:proof12}\\
        \Rightarrow &Z\perp \!\!\! \perp D|I
    \end{align}}
\textcolor{black}{
where \textbf{Step (a): Conditional Independence and Redundancy Elimination. }
The transition from Eq. (\ref{eq:proof11}) to Eq. (\ref{eq:proof12}) arises from the conditional independence induced by the optimal parameterization \(\theta^*\). Specifically, since Eq. (\ref{eq:proof11}) asserts the equivalence \(P_{\theta^*}(z|i_A, d_B) = P_{\theta^*}(z|i_A, d_A)\) for arbitrary \(d_A, d_B \sim P(D)\), it implies that conditioning on \textbf{any} data instance (e.g., \(d_A\) or \(d_B\)) provides no additional information about \(z\) beyond the identifier \(i_A\).  
Formally, for the joint conditioning case \(P_{\theta^*}(z|i_A, d_A, d_B)\):  
1. \textbf{Redundancy of \(d_B\)}: Given the pair \((i_A, d_A)\), the additional condition \(d_B\) becomes redundant due to the equivalence in Eq. (\ref{eq:proof11}). Hence,  
   \[
   P_{\theta^*}(z|i_A, d_A, d_B) = P_{\theta^*}(z|i_A, d_A).
   \]  
2. \textbf{Symmetry}: By symmetry between \(d_A\) and \(d_B\), we simultaneously derive  
   \[
   P_{\theta^*}(z|i_A, d_A, d_B) = P_{\theta^*}(z|i_A, d_B).
   \]  
This establishes the equality chain \(P_{\theta^*}(z|i_A, d_A) = P_{\theta^*}(z|i_A, d_B)\) in Eq. (\ref{eq:proof12}).}

\textcolor{black}{
\textbf{Step (b): Marginalization via Total Probability.} 
The first equality in derivation (b) follows directly from the Law of Total Probability, expanding the conditional distribution by integrating over the data variable \(d\). The second equality leverages the redundancy property proven in Step (a): since \(P_{\theta^*}(z|i_A, d)\) remains invariant to the choice of \(d\), marginalizing over \(d \sim P(D)\) preserves the distributional equivalence, yielding  
\[
P_{\theta^*}(z|i_A) = \mathbb{E}_{d\sim P(D)}[P_{\theta^*}(z|i_A, d)] = P_{\theta^*}(z|i_A, d).
\]  
Here, the marginalization over \(d\) collapses to a single representative instance due to the uniformity guaranteed by Eq. (\ref{eq:proof11}).}
\textcolor{black}{Then the second equality in Eq.(\ref{eq:proof12}) holds because of the total probability theorem. This completes the proof.}
\end{proof}

 In what next, we describe our Contrastive Desensitization Network that solves Eq. (\ref{eq:denoisingreconstruction_DTV}) by minimizing its upper bound. After this, we plug the learnt representation in E.q.(~\ref{eq_objective_crossdomain}) to train a downstream forgery face detector.

\section{Contrastive Desensitization Network}

 \begin{figure*}[t]

    \setlength{\abovecaptionskip}{0.cm}
    \setlength{\belowcaptionskip}{-0.5cm}
    \centering
    \includegraphics[width=0.9\linewidth]{ 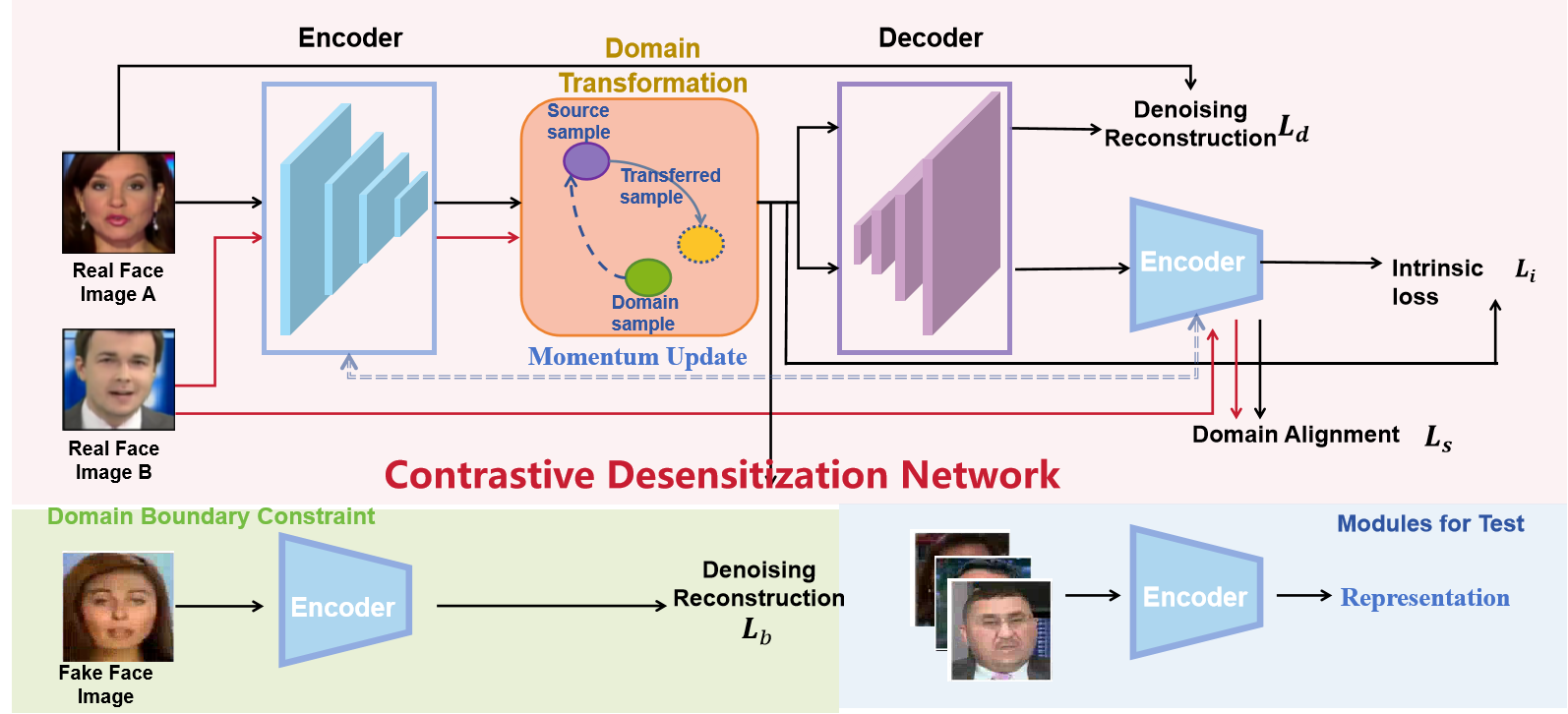}
    \caption{The overall architecture of the proposed CDN for face forgery detection.
To learn domain-invariant representations \( Z \) from given real face images \( X \). During the training phase of the CDN framework, the input image \( X \) is first processed by an encoder to extract its initial representation \( z \). Next, \( z \) is separated into intrinsic features \( I \) and domain-specific features \( D \) in the latent space. A domain transformation is then applied to mix \( I \) and \( D \), generating a new representation \( z_{\text{out}} \). Finally, \( z_{\text{out}} \) is passed through a decoder to reconstruct the original image, ensuring the removal of domain-specific noise while preserving intrinsic features. Three components to ensure this objective: 
Intrinsic and Domain Alignment for ensuring consistency across domains while retaining intrinsic features. Denoising Reconstruction to enhance the reliability of domain-invariant representations via decoder-based reconstruction}
    \label{fig:framework}
\end{figure*}
 In this section, we give a detailed description of the proposed CDN approach. The overall architecture is given in Figure \ref{fig:framework}.
 To learn a domain-invariant representation $Z$ for a given face image $X$, it is crucial to separate its intrinsic feature $I$ and domain-specific features $D$. 
 
 For this purpose, given two random real face samples from two different domains, we first extract their low-level visual features using an encoder, and then process them with a domain transformation operation. As shown in Figure \ref{fig:framework}, to ensure that the task of desensitization learning is feasible, three extra key components are equipped based on this representation, i.e., intrinsic/domain alignment and denoising reconstruction.

 The intrinsic/domain alignment and denoising reconstruction are used to model real human faces across different domains, which can be thought of as a hybrid rewarding mechanism that provides a feedback signal to our desensitization network. After learning, only the encoder module would be kept to yield new representation for a given unseen face image.

\subsection{Domain transformation}\label{sec_domaintransformation}
The first step of our CDN network is to perform a domain transformation $T$ for two random samples from different domains (but belong to the same real category.). In fact, the domain transform mixes the low-level visual features of the two samples, while yielding a new intrinsic representation in the same feature space.

In particular, for a given pair of low-level feature sets $z_A$ and $z_B$, extracted by an encoder from two real face images $x_A$ and $x_B$, respectively, we mix them based on their feature statistics~\cite{huang2017arbitrary}, i.e., the mean and standard deviation. as follows, 
\begin{align}
    z_{out} = \sigma_B\frac{z_A - \mu_A}{\sigma_A} + \mu_B
    \label{eq:domaintransformation}
\end{align}
where the $\mu_A,\mu_B,\sigma_A,\sigma_B$ are the feature statistics ($\mu$ is the mean and $\sigma$ is the standard deviation) calculated over $z_A$ and $z_B$. The transformation in Eq.(\ref{eq:domaintransformation}) could be seen as a domain normalization to $z_A$, which essentially aligns $z_{A}$ with the feature statistics of $z_B$, making the yielded feature $z_{out}$ has the same style as $z_B$. In ~\cite{zhou2023mixstyle}, this is thought as a feature augmentation procedure for representation learning, although we have a different interpretation for this (please see Section~\ref{subsec_features} for details). 

To ensure that $z_{out}$ preserves the intrinsic feature of $z_A$, we introduce two additional losses for further verification: i.e., intrinsic loss and domain alignment loss, described below. Let the encoder and the decoder in Figure \ref{fig:framework} be parameterized via $\theta$ and $\phi$ respectively. Given the transformed feature $z_{out}$, the intrinsic loss is defined as,
\begin{align}
    L_i = \|\theta(\phi(z_{out})) - z_{out}\|_2^2
    \label{eq:contentalignment}
\end{align}
Such a loss (also known as content loss) is widely used in textural synthesis~\cite{ulyanov2016texture,ulyanov2017improved}, and is beneficial to maintain the structural information of the reconstructed images. 

Another important aspect is the domain alignment between the pair of images, which can be defined as follows,

\begin{align}
    L_s &= \sum_{i=1}^L \|\mu(\theta_i(x_B)) - \mu(\theta_i(\phi(z_{out})))\|_2^2 \\
    &+ \sum_{i=1}^L \|\sigma(\theta_i(x_B)) - \sigma(\theta_i(\phi(z_{out})))\|_2^2
    \label{eq:domainalignment}
\end{align}
where $\mu$ and $\sigma$ are the feature statistics extracted via an MLP as mentioned before. The domain alignment loss in Eq.(\ref{eq:domainalignment}) allows us to align the domain information of $x_A$ with $x_B$~\cite{li2017demystifying,zhou2023mixstyle}.

\subsection{Learning to Desensitize}

To obtain a domain-invariant representation $z$ for a given real face image $x_A$, our idea is learning to desensitize based on the output $z_{out}$ of the domain transformation. In particular, as $z_{out}$ has been produced with the style of some other domain of $x_B$, to learn to remove such style information, what we need is simply to project it back to its original manifold where $x_A$ lies, as follows, 
\begin{align}
    L_d &= \|\phi(z_{out}) - x_A\|_2^2
    \label{eq:denoisingreconstruction}
\end{align}
In words, we learn to desensitize the domain style information of $x_B$ from $z_{out}$ with the help of a learned decoder $\phi$. We give the theoretical justification for this in the next section.

The denoising reconstruction in Eq.(\ref{eq:denoisingreconstruction}) is as,
\textcolor{black}{
\begin{align}
    \min_{\theta,\phi}\mathbb E_{z\sim P_\theta(z|i_A, d_B)} \|\phi(z) - x_A\|_2\label{eq:appendix21}
\end{align}}
And the probabilistic modeling of denoising reconstruction is,
\begin{align}
    \max_{\theta,\phi}\mathbb E_{z\sim P_\theta(z|i_A, d_B)} \log P_\phi(i_A, d_A|z)\label{eq:appendix22}
\end{align}
Before the derivation, we need to assume the likelihood $P_\phi(i_A, d_A|z) = P_\phi(x_A|z)$ is isotropic Gaussian ($\Sigma = \lambda I_{K\times K}$, $K$ is the dimension of $x_A$, $\lambda$ is the eigenvectors of variance matrix). Then we give the derivation from Eq.(\ref{eq:appendix22}) to Eq.(\ref{eq:appendix21}) as follows,
\begin{align}
    & \max_{\theta,\phi}\mathbb E_{z\sim P_\theta(z|i, d_B)} \log P_\phi(i, d_A|z)\\
    \Leftrightarrow&\max_{\theta,\phi}\mathbb E_{z\sim P_\theta(z|i, d_B)} \log \big[\frac{1}{\sqrt{(2\pi)^K|\Sigma|}}\exp{((\phi(z) - x_A)^T\Sigma^{-1}(\phi(z) - x_A))}\big]\\
    \Leftrightarrow&\min_{\theta,\phi}\mathbb E_{z\sim P_\theta(z|i, d_B)} \|\phi(z) - x_A\|_2
\end{align}

\subsection{Domain Boundary Constraint. }
To prevent {\it over-generalization}, it is necessary to constrain the boundary of the domain, maintaining a sufficient margin between the real image domain and the fake image domain. For this, we utilize a contrastive loss. Let $z_{out}^i$ and $z_{out}^j$ denote the domain-invariant representation of the real image $x_S^i$ and $x_S^j$, respectively, and $z_{f}^j$ denote the representation of the fake image $x_F^j$. Then the contrastive loss is defined as: 
\begin{align}
    L_{b} = &\sum_{N_{r}N_r}\sum_{i,j\in R} Dis(z_{out}^i,z_{out}^j)- \sum_{N_{r}N_f}\sum_{i\in R, j\in F} Dis(z_{out}^i,z_{f}^j)\label{eq:contrastiveloss}
\end{align}
Where \( R \) and \( F \) represent the sets of real and fake images, and \( N_r \) and \( N_f \) denote their respective sizes. The function \( Dis(x, y) \) is a cosine distance-based metric, expressed as:
\begin{align}
    Dis(x,y) = \frac{1}{2}\cdot [1 - \frac{x}{\|x\|_2}\cdot \frac{y}{\|y\|_2}]
\end{align}
\textcolor{black}{where $x,y$ are two arbitrary vectors, and $\|\cdot\|_2$ is the 2-norm operator.}

\subsection{Theoretical Justification for the Proposed Method}
Next, we show that under certain mild conditions which will be explained later, the proposed CDN approach solves 
Eq. (\ref{eq:denoisingreconstruction_DTV}) by minimizing its upper bound.  In particular, with the help of Eq.(\ref{eq:contentalignment}) and Eq.(\ref{eq:domainalignment}), the domain transformation described in Section~\ref{sec_domaintransformation} implements the following transformation: $T(F(x_A), F(x_B)) = (i_A, d_B)$ by Eq.(\ref{eq:domaintransformation}). That is, it essentially constructs a new hybrid sample $(i_A, d_B)$ by perturbing $i_A$ from its manifold with domain noise $d_B$. Hence what the denoising reconstruction objective Eq.(\ref{eq:denoisingreconstruction}) does is simply learning to recover from such perturbation so as to return to the manifold of the intrinsic features where $i_A$ originally lies, i.e., learning to desensitize domain shift.

More formally, using the language of probabilistic modeling, the denoising reconstruction objective (Eq.(\ref{eq:denoisingreconstruction})) can be reformulated as,
\begin{align}
\max_{\theta,\phi}\mathbb E_{z\sim P_\theta(z|i_A, d_B)} \log P_\phi(i_A, d_A|z)\label{eq:denoisingreconstruction_prob}
\end{align}
where $i_A$ means the intrinsic feature of sample $A$, $d_A$ is sample $A$'s domain information as is introduced in Eq.(\ref{eq:domaintransformation}).

The following Theorem \ref{theorem:denoising} builds the connection between the denoising reconstruction and domain desensitization explicitly. For the sake of simplicity, we assume that the decoder $\phi$ is fixed.

\begin{theorem}\label{theorem:denoising}
    Under the assumption that the probability density of the hidden space is commonly larger than the original sample space, i.e., $\forall z,\theta,P_\theta(z|i_A,d_A)\geq P_\phi(i_A,d_A|z)$. Then maximizing the denoising reconstruction term in Eq.(\ref{eq:denoisingreconstruction_prob}), i.e., 
    \begin{align}
       \max_{\theta}\mathbb E_{z\sim P_\theta(z|i_A, d_B)} \log P_\phi(i_A, d_A|z)\label{eq:theorem11}
    \end{align}
    is equivalent to minimizing the upper bound of the following objective,
    \begin{align}
        D_{KL}\big(P_\theta(z|i_A, d_B)\big\|P_\theta(z|i_A, d_A)\big)\label{eq:theorem12}
    \end{align}
    , where $D_{KL}$ is the KL-divergence.
\end{theorem}
Eq.\(\eqref{eq:theorem11}\) captures intrinsic features shared across domains, treating domain-specific information as noise, while Eq.\(\eqref{eq:theorem12}\) explicitly denoises to obtain domain-invariant representations. \textcolor{black}{To further explain how Eq.\(\eqref{eq:theorem12}\) guarantees the extraction of domain-invariant representations, we plot the diagram in Figure \ref{fig:kl}, where the minimization of two-side (forward and backward) KL divergence could align the two latent distributions totally. Then the merged representation would be recognized only by the intrinsic features, i.e., be domain-invariant.
\begin{figure}[h]
    \centering
    \includegraphics[width=0.8\linewidth]{ 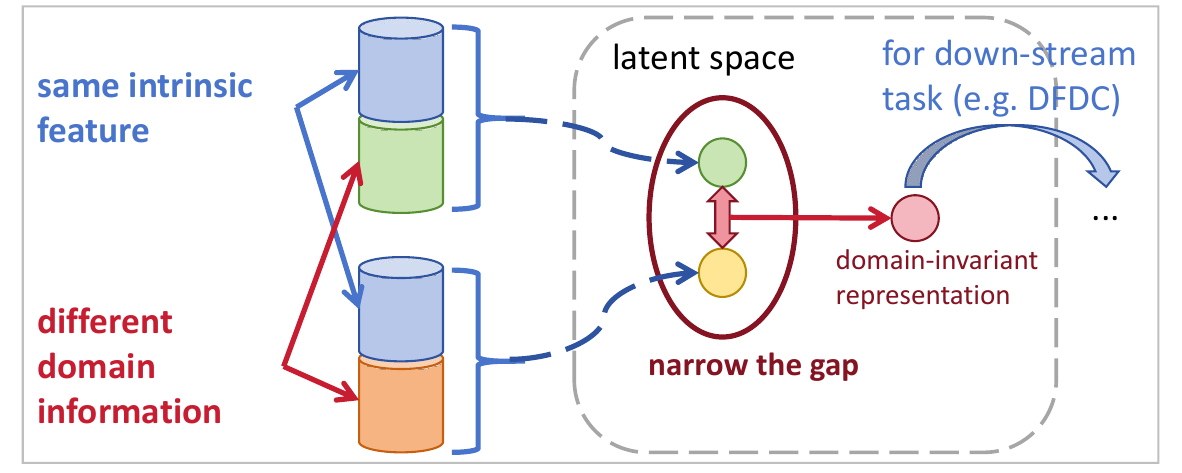}
       \caption{Diagram of the domain-invariant objective.}
       \label{fig:kl}
\end{figure}}

The two objectives are theoretically connected, as shown in the following proof, where \(\eqref{eq:theorem11}\) is transformed into \(\eqref{eq:theorem12}\) through minimization of the negative log-likelihood and scaling.
\begin{proof}

The proof consists of the following steps: First, we transform Eq.(\ref{eq:theorem11}) into a minimization of the negative log likelihood; then, based on assumptions and by introducing an additional negative term for scaling, it ultimately takes the form of the KL divergence as shown in Eq.(\ref{eq:theorem12}). The detailed derivation is as follows:
    \begin{align}
        &\max_{\theta}\mathbb E_{z\sim P_\theta(z|i_A, d_B)} \log P_\phi(i_A, d_A|z)\\
        \Leftrightarrow & \min_{\theta} - \mathbb E_{z\sim P_\theta(z|i_A, d_B)} \log P_\phi(i_A, d_A|z)
    \end{align}
    Then,
    \begin{align}
        &-\mathbb E_{z\sim P_\theta(z|i_A, d_B)} \log P_\phi(i_A, d_A|z)\\
        =& \mathbb E_{z\sim P_\theta(z|i_A, d_B)} \log \frac{1}{P_\phi(i_A, d_A|z)}\\
        \geq^{(a)} & \mathbb E_{z\sim P_\theta(z|i_A, d_B)} \log \frac{1}{P_\phi(z|i_A, d_A)}\\
        \geq^{(b)} & \mathbb E_{z\sim P_\theta(z|i_A, d_B)} \log \frac{1}{P_\phi(z|i_A, d_A)}+ \mathbb E_{z\sim P_\theta(z|i_A, d_B)} \log P_\phi(z|i_A, d_B)\\
        =& D_{KL}\big(P_\theta(z|i_A, d_B)\big\|P_\theta(z|i_A, d_A)\big)
    \end{align}
    Note that the inequality $(a)$ holds because of the assumption that $P_\theta(z|i_A,d_A)\geq P_\phi(i_A,d_A|z)$. The inequality $(b)$ holds because $P_\phi(z|i_A, d_B)<1$, so that $E_{z\sim P_\theta(z|i_A, d_B)}\log P_\phi(z|i_A, d_B) < 0$.
    This completes the proof.
\end{proof}

The theorem reveals that solving the denoising objective Eq.(\ref{eq:denoisingreconstruction}) approaches from above the optimal solution of an alternative problem given in Eq.(\ref{eq:denoisingreconstruction_DTV}), which is in turn equivalent to seeking a robust representation against domain changes.


\begin{remark}

\textcolor{black}{
The essence of Theorem \ref{theorem:denoising} lies in its dual capability: \textbf{denoising-driven expectation maximization} and \textbf{implicit latent space alignment across domains}. In practical implementations, these theoretical properties translate into two critical operational advantages:  
\begin{enumerate}
    \item \textbf{Renoising-driven expectation maximization.} By constraining the KL-divergence between latent distributions under different data conditions ($d_A$ vs. $d_B$), the model learns to extract domain-invariant representations from semantically similar samples contaminated by domain-specific variations (e.g., imaging artifacts in medical devices or lighting differences in surveillance footage). This mechanism effectively mitigates the domain shift problem, where traditional models degrade due to distributional discrepancies between training and deployment environments.
    \item \textbf{Implicit latent space alignment across domains.} Crucially, the alignment is achieved without requiring explicit domain labels - the optimization solely relies on denoising reconstruction objectives. This label-agnostic nature makes the theorem particularly valuable for: Deepfake detection: Aligning latent spaces of manipulated and authentic media across diverse forgery techniques (e.g., FaceSwap vs. DeepFaceLab artifacts); Low-resource scenarios: Applications where domain annotation is impractical (e.g., cross-lingual speech processing, multi-center medical imaging); Dynamic environments: Situations with continuously evolving domains (e.g., adapting to new camera sensors in autonomous vehicles)  
\end{enumerate}
The implicit alignment occurs through the theorem's probabilistic coupling - maximizing $P_\phi(i_A,d_A|z)$ under noisy $d_B$ inputs forces the encoder to discard domain-specific noise patterns while preserving semantic content in $z$. This creates a "purified" latent subspace resilient to both explicit adversarial perturbations and natural domain variations.}
    
\end{remark}

Before ending this section, we would like to give some intuitive explanation on the assumed conditions of Theorem \ref{theorem:denoising}, i.e., $\forall z,\theta,P_\theta(z|i_A,d_A)\geq P_\phi(i_A,d_A|z)$. Actually, it is not so restrictive as it looks - it basically says that  we should project samples into the latent space in such a way that facilitates their reconstruction, given a fixed decoder $\phi$ - a condition that is not so hard to satify in practice. \textcolor{black}{
In specific applications and implementations, the condition often holds naturally. For instance, in the context of a variational autoencoder (VAE), the two distributions are typically modeled as Gaussian: $P_\theta(z|i_A,d_A)=\mathcal{N}(z; \mu_\theta, \sigma_\theta^2), P_\phi(i_A,d_A|z)=\mathcal{N}(i_A,d_A; \mu_\phi, \sigma_\phi^2)$. To simplify, we denote $x=(i_A,d_A)$ and we can compare the two distributional values by calculating their log difference as,
\begin{align}
    \Delta &= \log P_\theta(z|x) - \log P_\phi(x|z)\\
    &= \underbrace{\frac{n}{2} \log(2\pi\sigma_\phi^2) - \frac{m}{2} \log(2\pi\sigma_\theta^2)}_{\text{variance term}} + \underbrace{\frac{\|x - \mu_\phi\|^2}{2\sigma_\phi^2} - \frac{\|z - \mu_\theta\|^2}{2\sigma_\theta^2}}_{\text{bias term}}\nonumber
\end{align}
where $n$ is the dimension of $x$, while $m$ is the dimension of $z$. 
In the training of VAE, the reconstruction loss like Eq.(\ref{eq:denoisingreconstruction}) tends to minimize $\frac{\|x - \mu_\phi\|^2}{2\sigma_\phi^2}$, so enhancing the $\sigma_\phi^2$, hence in many cases, we would have $\sigma_\phi^2>\sigma_\theta^2$. Then the variance term would always be positive. In the application of DFDC, the input samples are commonly high-dimensional images, i.e., $n>>m$, the reconstruction loss $\|x - \mu_\phi\|^2$ would be larger than $\|z - \mu_\theta\|^2$ commonly, in which case the bias term would tend to be positive. In conclusion, we can assert that the assumption holds, and $\Delta>0$ in scenarios where the input data is high-dimensional and the method is implemented within a VAE framework.
}

\begin{table*}[htbp]
    \centering
    \setlength{\abovecaptionskip}{0.cm}
    \setlength{\belowcaptionskip}{-0.5cm}
    \resizebox{\linewidth}{!}{
    \begin{tabular}{lcccccccccccc}
        \toprule
        Methods &\multicolumn{2}{c}{FF++(c23)} & \multicolumn{2}{c}{FF++(c40)} & \multicolumn{2}{c}{Celeb-DF}& \multicolumn{2}{c}{WildDeepfake} & \multicolumn{2}{c}{DFDC} & \multicolumn{2}{c}{\textbf{Average}}\\
		     &ACC  & AUC   &ACC  & AUC  &ACC  & AUC  &ACC  & AUC &ACC  & AUC  &ACC  & AUC \\
        \midrule
        Xception\cite{Xception} &95.73 &96.30 &86.86 &89.30 &97.90 &99.73 &77.25 &86.76 &79.35 &89.50 & 87.42 &92.32 
\\
        $F^{3}$-Net\cite{DFDC_Frequency1} &97.52 &98.10 &90.43 &93.30 &95.95 &98.93 &80.66 &87.53 &76.17 &88.39 &88.15 &	93.25 
\\
         Add-Net\cite{WildDeepfake} &96.78 &97.74 &87.50 &91.01 &96.93 &99.55 &76.25 &86.17 &78.71 &89.85 & 87.23 	&92.86 
\\
        MultiAtt\cite{MultiAtt}  &97.60 &99.29 &88.69 &90.40 &97.92 &99.94 &82.86 &90.71 &76.81 &90.32 & 88.78 	&94.13 
\\
       
        RFM\cite{RFM}    &95.69 &98.79 &87.06 &89.83 &97.96 &99.94 &77.38 &83.92 &80.83 &89.75 & 87.78 	&92.45 
\\
        RECCE\cite{RECCE}  &97.06 &99.32 &91.03 &95.02 &98.59 &99.94 &83.25 &92.02 &81.20 &91.33 & 90.23 	&95.53 
\\
        ITA-SIA\cite{2022_ITA-SIA} &97.64 &99.35 &90.23 &93.45 &98.48 &99.96 &83.95 &91.34 &– &– & - & -\\
        DisGRL\cite{shi2023discrepancy} &\textbf{97.69} &99.48 &91.27 &95.19 &98.71 &99.91 &84.53 &93.27 &82.35 &92.50 &90.91 &	96.07 
\\

FIC\cite{bai2024towards}&97.14&99.29&91.27 &92.30 &-&-&-&-&-&-&-&-
\\
        \textcolor{black}{MDDE\cite{qiu2024multi}} &\textcolor{black}{97.30} &\textcolor{black}{\textbf{99.49}}&\textcolor{black}{90.67} & \textcolor{black}{95.21} &\textcolor{black}{98.63} &\textcolor{black}{99.97} &\textcolor{black}{84.46} &\textcolor{black}{91.93}&\textcolor{black}{84.91} & \textcolor{black}{91.24} & \textcolor{black}{91.19} & \textcolor{black}{95.57}\\
        CDN(Ours) &97.57 \textcolor{black}{$\pm0.8$} &99.29 \textcolor{black}{$\pm0.4$}&\textbf{91.54} \textcolor{black}{$\pm0.7$}&\textbf{95.30} \textcolor{black}{$\pm0.1$}&\textbf{99.94} \textcolor{black}{$\pm0.1$}&\textbf{99.99} \textcolor{black}{$\pm0.1$}&\textbf{85.21}\textcolor{black}{$\pm0.4$}&\textbf{93.41}\textcolor{black}{$\pm0.3$}&\textbf{86.87} \textcolor{black}{$\pm1.4$}& \textbf{93.24}\textcolor{black}{$\pm0.7$}
        & \textbf{92.23}&	\textbf{96.24}\\
        \bottomrule
    \end{tabular}}
    \caption{ Comparative performance for various methods with intra-dataset evaluation. \textcolor{black}{The standard deviations of our method's results are calculated on 4 random seeds.}}
    \label{intra-dataset}
\end{table*}

\section{Experiments}

\subsection{Experimental Settings}

\paragraph{Datasets.} Our experiments are conducted on four challenging datasets specifically designed for deepfake detection, including FaceForensics++ (FF++)~\cite{FF++}, CelebDF~\cite{Celeb-DF}, WildDeepfake (WDF)~\cite{WildDeepfake}  and DFDC~\cite{DFDC}.

As the most widely used dataset, FF++\cite{FF++} contains 1000 real videos collected from Youtube and 4000 forgery videos from four subsets of different face forgery techniques, \textit{i.e} Deepfakes (DF)\cite{DF}, Face2Face (F2F)\cite{Face2Face}, FaceSwap (FS)\cite{FS}, and NeuralTextures (NT)\cite{NeuralTextures}. 
Among them, Deepfakes (DF)\cite{DF} and FaceSwap (FS)\cite{FS} belong to face replacement forgery, and Face2Face (F2F)\cite{Face2Face}and NeuralTextures (NT)\cite{NeuralTextures} belong to facial expression attribute forgery. In terms of compression method, the data set provides two different compression levels: c23(constant rate quantization parameter equal to 23) and c40(the quantization parameter is set to 40).

The CelebDF~\cite{Celeb-DF} dataset contains 480 real videos and 795 forged videos. The real videos are sourced from YouTube, with an average length of 13 seconds and a frame rate of 30 fps. The authors have made improvements including enhancing the resolution, implementing facial color transformation algorithms, blending the boundaries of synthetic faces, and reducing the jitter in the synthesized videos to the visual quality of the forged videos. 

The DFDC~\cite{DFDC} is the official dataset for the Deepfake Detection Challenge. It comprises a total of 119,196 videos, with a ratio of genuine to forged videos of approximately 1:5. The original videos were recorded by actors, with an average length of around 10 seconds. This dataset encompasses a broad range of video resolutions and features diverse and complex scenarios, including dark backgrounds with Black subjects, profile views, people in motion, strong lighting conditions, and scenes with multiple individuals.

The WildDeepfake~\cite{WildDeepfake} is a more challenging dataset which consists of 7,314 face sequences extracted from 707 deepfake videos collected completely from the internet. 

Consistent with previous works\cite{RECCE}, this paper employs the same data preprocessing methods and test set selection to ensure a fair and objective comparison.
\paragraph{Inference Details}
During the training process, in order to achieve desensitization of the style features of real images, two features are randomly selected, one as the source domain and the other as the target domain. During the inference process, we input the first layer features of the encoder in the auto-encoder and the first and second layer features of the decoder into the downstream task for the final prediction.
\paragraph{Implementation Details}
We implement the proposed Contrastive Desensitization Network (CDN) within a general face forgery detection framework, where the produced domain-invariant representation is fed into the downstream task module for final forgery detection. In particular, our downstream task module adopts two sequential process steps, i.e., information aggregation, multi-scale graph reasoning and attention-guided feature fusion, the details of which can be found in Appendix \ref{appendix:implementation}. This pipeline has been proven effective for face forgery detection in many previous works~\cite{RECCE,shi2023discrepancy,shuai2023locate}. Let the loss for classification be $L_{cls}$, which can be any binary classification loss function, such as Binary Cross Entropy (BCE). Then the whole loss function of our system is as follows,
\begin{align}
    L = L_{cls} + \lambda_1 L_{d} + \lambda_2 L_{i} + \lambda_3 L_{s}
\end{align}\label{Loss}
where the denoising reconstruction loss $L_d$ (Eq.(\ref{eq:denoisingreconstruction})), the intrinsic alignment loss $L_i$ (Eq.(\ref{eq:contentalignment})), and the domain alignment loss $L_s$ (Eq.(\ref{eq:domainalignment})) are included. And $\lambda_1,\lambda_2$ and $\lambda_3$ are three coefficients balancing the relative importance of these losses, whose values are set by cross-validation. We train our model with a batch size of 16, the Adam~\cite{kingma2014adam} optimizer with an initial learning rate of 2e-4 and a weight decay of 1e-5. A step learning rate scheduler is used to adjust the learning rate. Two NVIDIA 3090Ti GPUs are used in our experiments. We empirically set the hyperparameter of eq.\ref{Loss} as $\lambda_1$= 0.1,$\lambda_2$= 0.1,$\lambda_3$= 0.1. The second encoder (i.e., the one on the right half of Fig.\ref{fig:framework}, used for loss evaluation) is trained using momentum update with momentum value set to be 0.999 as recommended in \cite{he2020momentum}

\paragraph{Implementation CDN to Downstream Task}\label{appendix:implementation}
In this section, we introduce the details of the implementation of the  Contrastive Desensitization Network(CDN).
Briefly, we take the Xception\cite{Xception} as the backbone, then apply the Domain Transformation module in convolutional blocks of the entry flow. After that, the transformed output feature is reconstructed through a decoder similar to the entry flow's structure.
The feature of encoder-decoder can be used in several full architectures such as RECCE\cite{RECCE} and D\cite{shi2023discrepancy}
\paragraph{Evaluation Metrics} 

\subsection{Intra-dataset Evaluation}\label{appendix:intra-dataset}

To evaluate the baseline performance of the proposed method to detect forgery face images, we conducted a series of intra-dataset experiments in which the test set is sampled from the same dataset as that used for training. We compared the proposed methods with several closely related state of the art face forgery detection methods, including RECCE\cite{RECCE}, ITA-SIA\cite{2022_ITA-SIA} and so on, by following the corresponding evaluation protocols defined over each datasets. 

Since our method performs forgery detection at the image-level and does not introduce any spatiotemporal features, we only compare the image-level with competitve method and do not include the video-level method such as RealForensics\cite{RealForensics}, AltFreezing\cite{AltFreezing} and CoReST\cite{CoReST} in this article.
\begin{figure}[h]
    \setlength{\abovecaptionskip}{0.cm}
    \setlength{\belowcaptionskip}{-0.2cm}
    \begin{minipage}{0.5\linewidth}
    \centering
    \includegraphics[width=1\linewidth]{ 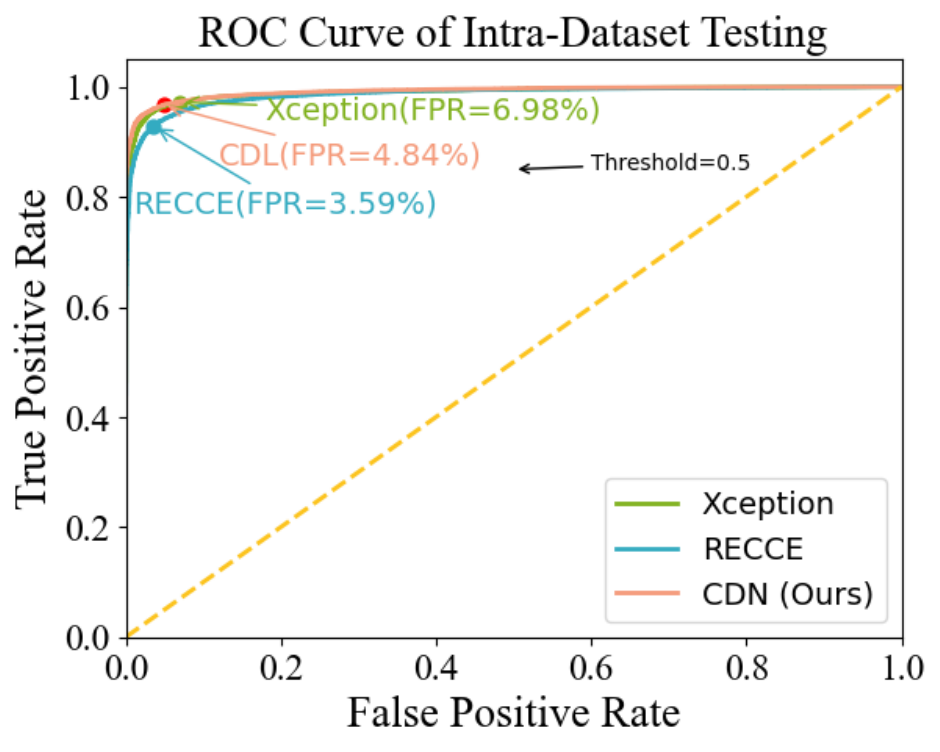} 
    \centerline{(a) Intra-Evaluation}
    \end{minipage}\hfill
    \begin{minipage}{0.5\linewidth}
    \centering
    \includegraphics[width=1\linewidth]{ 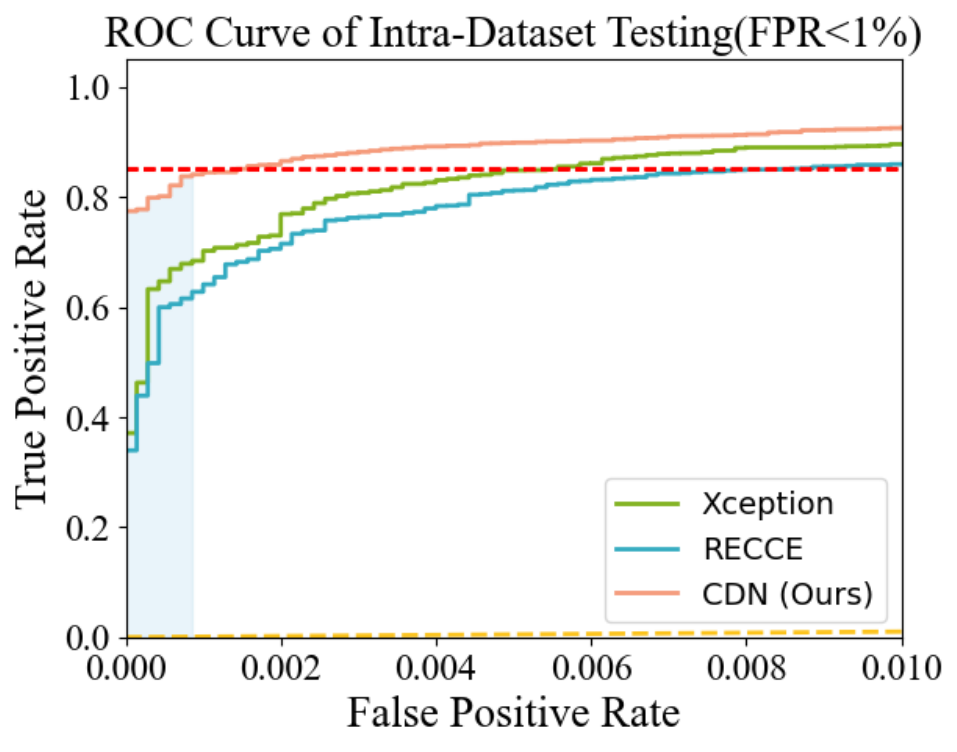} 
    \centerline{(b) Cross-Evaluation}
    \end{minipage}
    \caption{The ROC curves of the compared intra-evaluation and cross-manipulation evaluation methods.}
    \label{false_alarm_roc_intra}
\end{figure}

Table \ref{intra-dataset} gives the results. It demonstrates that the proposed CDN method is comparable or superior to several other approaches among most of the standard benchmark datasets in terms of both ACC and AUC scores, despite that in our method the face representation is learnt without using the guide of any knowledge about what forgery face images look like. In particular, on the high-quality datasets Celeb-DF and DFDC, our method outperforms the SOTA methods RECCE by 1.35\% and  5.67 \% respectively in terms of on ACC score.

Based on the public open resources available \cite{Xception}\cite{RECCE}, we made a detailed comparison with Xception\cite{Xception} and RECCE\cite{RECCE}, as both are SOTA and popular face forgery detection methods and are closely related to our method in terms of methodology. Figure \ref{false_alarm_roc_intra}(a) gives ROC curves of the compared methods. First, from the Figure \ref{false_alarm_roc_intra}(b),
We assume that the minimum TPR requirement for a detector is 85\%, which is what the red line means.
We see that under the requirement of \textit{FPR} below 0.1\%, the TPR performance of Xception and RECCE degraded significantly to 68.39\% and 62.70\%, respectively, although both of them achieve Acc score higher than 95.0\% on this dataset of FF++(c23). By contrast, the TPR of our CDN maintains 84.4\% under this setting. Furthermore, if we fix beforehand an acceptable target TPR performance (e.g., 85\%, as indicated with red line in the figure), we see that the proposed CDN method achieves much lower \textit{FPR} value (0.14\%) than both Xception (0.84\%) and RECCE (0.53\%), indicating the effectiveness of our method in reducing the false alarm while maintaining a high true forgery face detection rate.

\begin{figure}[htbp]
    \centering
    \includegraphics[width=0.8\linewidth]{ 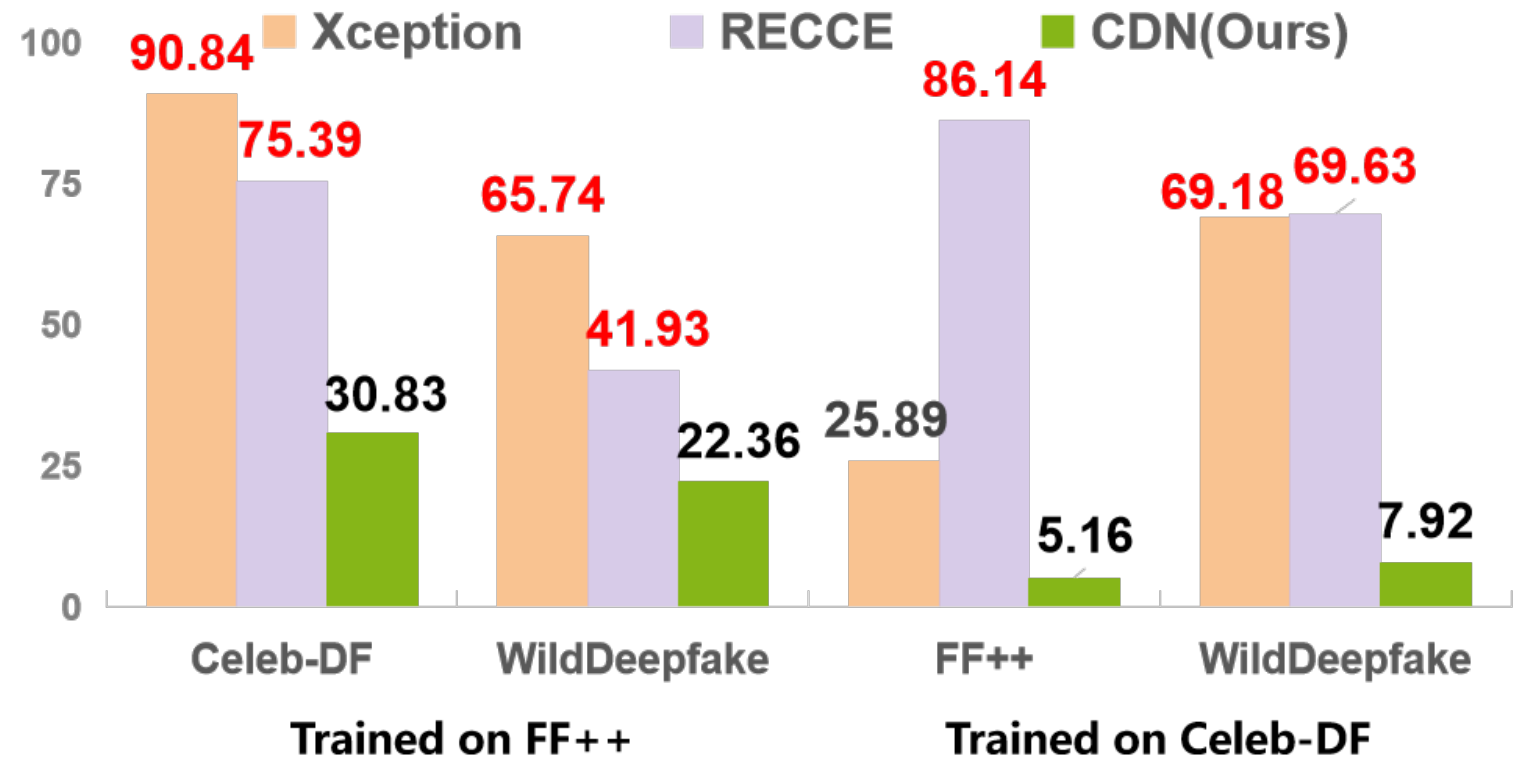}
       \caption{False Alarm Rate(FPR) ($\downarrow$) when cross-dataset testing among dataset FF++, Celeb-DF(CDF), WildDeepfake(WDF). The left two are trained on FF++, and the right two are on CDF.}
       \label{cross_dataset-FAP}
\end{figure}
\subsection{Cross-Domain Evaluation}\label{section_Cross_domain}

\paragraph{Cross-Dataset Evaluation}

To explore the generalization of our method on unseen datasets compared with recent general face forgery detection methods, we conducted a series of experiments on more challenging cross-dataset evaluation. In particular, we train our model on FF++(c40)~\cite{FF++} and test it on other three datasets: DFDC~\cite{DFDC}, Celeb-DF~\cite{Celeb-DF} and WDF~\cite{WildDeepfake}. Table \ref{cross_domain}(a) gives the results, from which one can see that the proposed CDN method consistently performs better than the compared method.  To investigate the performance of false positives, we also compare our method with RECCE~\cite{RECCE}, DisGRL~\cite{shi2023discrepancy} and Xception~\cite{Xception}. Figure \ref{cross_dataset-FAP} gives the results. It illustrates that our method outperforms the other two methods by a large margin in reducing \textit{FPR}. In particular, when testing on FF++~\cite{FF++} and WildDeepfake~\cite{WildDeepfake} (trained on Celeb-DF~\cite{Celeb-DF}), our method alleviates the \textit{FPR} by \textbf{80.98\%} and \textbf{61.71\%} respectively compared to RECCE~\cite{RECCE}. 

\begin{table*}[h]
\centering
    \setlength{\abovecaptionskip}{0.cm}
    \setlength{\belowcaptionskip}{-0.5cm}
\resizebox{0.7\linewidth}{!}{\begin{tabular}{lcccccc}
        \toprule
        Methods &\multicolumn{2}{c}{Celeb-DF} & \multicolumn{2}{c}{WildDeepfake} & \multicolumn{2}{c}{DFDC} \\
		     & AUC & EER & AUC & EER& AUC & EER \\
        \midrule
        Xception\cite{Xception} &61.80 &41.73 &62.72 &40.65 &63.61 &40.58 \\
        $F^{3}$-Net &61.51 &42.03 &57.10 &45.12 &64.60 &39.84\\
        MultiAtt\cite{MultiAtt}    &67.02 &37.90&59.74 &43.73  &68.01 &37.17\\
        Add-Net\cite{WildDeepfake} &65.29 &38.90 &62.35 &41.42 &64.78 &40.23\\
        RFM\cite{RFM} &65.63 &38.54 &57.75 &45.45 &66.01 &39.05\\
        RECCE\cite{RECCE} &68.71 &35.73 &64.31 &40.53 &69.06 &36.08\\
        \textcolor{black}{MDDE\cite{qiu2024multi}} &\textcolor{black}{68.80}&\textcolor{black}{35.68} &\textcolor{black}{70.92}&\textcolor{black}{\textbf{35.15 }}&\textcolor{black}{66.83} &\textcolor{black}{36.83}\\
      
         CDN(Ours) &\textbf{70.73}\textcolor{black}{$\pm0.6$}&\textbf{34.66}\textcolor{black}{$\pm1.8$} &\textbf{71.26}\textcolor{black}{$\pm2.1$}&\textbf{35.20 }\textcolor{black}{$\pm4.3$}&\textbf{70.21}\textcolor{black}{$\pm2.7$} &\textbf{35.08}\textcolor{black}{$\pm5.8$}\\
        \bottomrule
    \end{tabular}}
        
        \caption{Cross-dataset evaluation in terms of AUC ↑ (\%) and EER ↓ (\%), where the model is trained on FF++ (LQ) but tested on Celeb-DF, WildDeepfake, and DFDC. \textcolor{black}{The standard deviations of our method's results are calculated among 4 random seeds.}}
\end{table*}\label{cross_domain_}

\paragraph{Cross-Manipulation Evaluation}

To further evaluate the generalization among different manipulated manners, we conduct the fine-grained cross-manipulation evaluation by training the network on FF++(c40)~\cite{FF++} with fake images from one of Deepfakes (DF), Face2Face (F2F), FaceSwap (FS), and NeuralTextures (NT) while testing its performance on the remaining three datasets, of which the results are given in Table \ref{cross_manupulation_excluded1}. We observe that our CDN generally outperforms the competitors in most cases, including both intra-manipulation (diagonal of the table) results and cross-manipulation. Furthermore, Figure \ref{false_alarm_roc_intra}(b) gives the detailed ROC curves of several methods. From this one can see that our method has a higher precision (true positive rate) than the compared methods under the same false alarm rate, meanwhile, it also delivers lower false alarm rate under any given precision.

\begin{table}[h]
\centering
\small
    \setlength{\abovecaptionskip}{0.cm}
    \setlength{\belowcaptionskip}{-0.2cm}
\resizebox{\linewidth}{!}{
    \begin{tabular}{lccccccc}
        \toprule
        Methods  & Train & DF & F2F & FS & NT &C.Avg\\
        \midrule
        Xception\cite{Xception}  & &\cellcolor[gray]{0.9}99.41	&56.05	&49.93	&66.32	 &57.43\\
        RECCE\cite{RECCE}  &   & \cellcolor[gray]{0.9}99.65&70.66	&74.29	&67.34	&70.76 \\
        DisGRL\cite{shi2023discrepancy} &DF &	\cellcolor[gray]{0.9}99.67&	71.76 	&\textbf{75.21} 	&68.74	&71.90 \\
        FIC\cite{bai2024towards}	& &\cellcolor[gray]{0.9}\textbf{99.47}&77.39&69.06&68.51&71.65\\
        CDN(Ours)	& &\cellcolor[gray]{0.9}\textbf{99.65}&72.38	&71.68	&\textbf{79.51}&\textbf{74.52}\\

        \midrule
        Xception\cite{Xception}   & &68.55	&\cellcolor[gray]{0.9}98.64	&50.55	&54.81	& 57.97\\
        RECCE\cite{RECCE}  &&75.99	&\cellcolor[gray]{0.9}98.06	&64.53	&72.32	& 70.95\\
        DisGRL\cite{shi2023discrepancy}	  &F2F    &75.73  &\cellcolor[gray]{0.9}98.69 	&65.71 		&71.86&71.10  \\
         FIC\cite{bai2024towards}	&   &78.07&\cellcolor[gray]{0.9}98.27&67.58&74.01&73.22\\
          CDN(Ours)	&   &\textbf{85.86}	&\cellcolor[gray]{0.9}\textbf{98.93}	&\textbf{66.09}&74.68 &\textbf{75.54}\\
        \midrule
        Xception\cite{Xception}   &  &49.89	&54.15	&\cellcolor[gray]{0.9}98.36	&50.74&51.59	 \\
        RECCE\cite{RECCE}&  &82.39	&64.44	&\cellcolor[gray]{0.9}98.82	&56.70	 &67.84\\
        DisGRL\cite{shi2023discrepancy}	 & FS  &82.73	&64.85 	&\cellcolor[gray]{0.9}99.01 	&56.96 &	68.18 \\
        FIC\cite{bai2024towards}&  &81.47	&65.28
        &\cellcolor[gray]{0.9}98.93&60.63&69.13\\
       CDN(Ours)   &  &\textbf{84.13}	&\textbf{66.38} &\cellcolor[gray]{0.9}\textbf{99.07}&\textbf{61.07}&\textbf{70.53}\\
        \midrule
        Xception\cite{Xception}   &  &50.05	&57.49	&50.01	&\cellcolor[gray]{0.9}\textbf{99.88}	 &52.52\\
        RECCE\cite{RECCE}&  & 78.83 &80.89 &63.70 &\cellcolor[gray]{0.9}93.63&74.47\\
        DisGRL\cite{shi2023discrepancy}	& NT &80.29 	&\textbf{83.30}	&65.23	&\cellcolor[gray]{0.9}94.10 &76.27	 \\
          FIC\cite{bai2024towards}&  &83.81	&78.60	&63.88	&\cellcolor[gray]{0.9}92.42	& 75.43\\
        CDN(Ours)	&  &\textbf{88.44}	&82.72	&\textbf{65.67}	&\cellcolor[gray]{0.9}96.27	& \textbf{78.94}\\
        \bottomrule
    \end{tabular}}
    \caption{Cross-manipulation evaluation in terms of AUC (\%), where intra-domain performance shown in diagonal, four image manipulation approaches in FF++ (i.e., DeepFakes (DF), Face2Face (F2F), FaceSwap (FS), and NeuralTextures (NT)) are shown in a separate column, and the last column is the average of cross-manipulation evaluations.} 
        \label{cross_domain}
\end{table}
\paragraph{Multi-Source Manipulation Evaluation}
To investigate the generalizability of the proposed method in more realistic scenarios, where the forgery data may come from different manipulation sources, we conduct multi-source manipulation evaluation on the FF++(c40)~\cite{FF++} dataset, with the same settings as LTW~\cite{sun2021domain} and DCL~\cite{zhang2022graph}. Table \ref{cross_manupulation_excluded1} gives the results, showing that our approach outperforms them both in terms of AUC and ACC score. It is worth noting that although our approach does not incorporate graph reasoning or transformer structures like DCL\cite{zhang2022graph}, it still outperforms DCL\cite{zhang2022graph} in this evaluation, demonstrating its significant potential in the task of cross-domain forgery detection.
\begin{table}[h]
    \centering
    \scriptsize
    \begin{tabular}{lccccc}
        \toprule
         Methods& GID-DF & GID-F2F & GID-FS & GID-NT  \\
        \midrule
        MultiAtt\cite{MultiAtt} &66.8/– &56.5/– &51.7/– &56.0/– \\
        MLDG\cite{MLDG}& 67.2/73.1 &58.1/61.7 &58.1/61.7 &56.9/60.7 \\
        LTW\cite{sun2021domain}&  69.1/75.6 &65.7/72.4 &62.5/68.1 &58.5/60.8 \\
        DCL\cite{zhang2022graph}& 75.9/83.8 &67.9/75.1 &–/– &–/– \\
        CDN(Ours)&  \textbf{77.8}/\textbf{87.0} &\textbf{76.8}/\textbf{85.7} 
        &  \textbf{66.0}/\textbf{75.3}          &\textbf{67.6}/\textbf{76.7}\\
        \bottomrule
    \end{tabular}
    \caption{Multi-source evaluation results on ACC/AUC (\%).}
    \label{cross_manupulation_excluded1}
\end{table}
\subsection{\textcolor{black}{Real-World Evaluation}}\textcolor{black}{
    To validate the generalizability of our method, we constructed a new dataset by applying advanced deepfake techniques (e.g., Deepfakes\cite{DF}, Face2Face\cite{Face2Face}, SimSwap, and Diffusion models) to publicly available images of well-known individuals. This dataset simulates realistic forgery scenarios and includes a diverse range of facial manipulations. We evaluated our proposed method on this dataset and compared its performance with state-of-the-art baselines. The results are summarized in Table \ref{tab:real}
\begin{table}[htbp]
\centering
\small
\caption{Results (AUC) on the real-world scenario datasets.}
\begin{tabular}{lllll}\toprule
Method             & Ds1-df & Ds2-f2f & Ds3-sim & Ds4-dif \\ \midrule
ResNet          & 0.585  & 0.551   & 0.556   & 0.537         \\
Xception & 0.913  & 0.753   & 0.801   & 0.674         \\
MLDG               & 0.918  & 0.730   & 0.771   & 0.607         \\
CDN (ours)               & \textbf{0.936}  & \textbf{0.814}   & \textbf{0.847}   & \textbf{0.724}        \\
\bottomrule
\end{tabular}
\label{tab:real}
\end{table}
}
\subsection{\textcolor{black}{Computational Efficiency}}

\textcolor{black}{
To analyze the computational efficiency and parameter size of the model, we comparing the key indicators (Params, FLOPs, Pass Size, Params Size) of Xception, RECCE and our method (Ours), the advantages of the model and the direction of improvement are explained in detail. The specific data are as follows: }

      \begin{table}[h]
          \centering
          \begin{tabular}{c|ccccc}
          \toprule
             Model&Params  &  FLOPs&Pass Size&Params size
             \\
             \midrule
             Xception&20.809M&0.85G &74.10MB&79.38MB\\
             RECCE& 23.817M&2.27G& 111.51MB&90.86MB\\
             Ours  &23.818M&1.14G &113.76MB&90.86MB\\
               \bottomrule
          \end{tabular}
          \caption{\textcolor{black}{Comparative results of the computational efficiency and parameter size}}
          \label{tab:my_label}
      \end{table}
      
\textcolor{black}{
We evaluated the proposed CDN from three indicators: computational efficiency (FLOPs), parameter quantity (Params), and memory usage (Pass Size \& Params Size).
Among them, Xception is the backbone of the proposed model, and RECCE is the baseline of the proposed method. From the perspective of computational efficiency (FLOPs), the FLOPs of the proposed method (1.14G) is significantly lower than that of RECCE (2.27G), indicating that we have effectively reduced the computational complexity by introducing lightweight designs (such as dynamic sparse convolution and hierarchical feature reuse).}

\textcolor{black}{
From the perspective of parameter quantity (Params), the proposed method (23.818M) is close to that of RECCE (23.817M), but through parameter sharing and mixed precision training, the model avoids parameter expansion while maintaining performance.}

\textcolor{black}{
From the perspective of memory usage (Pass Size \& Params Size), the slightly higher Pass Size (113.76MB vs. 111.51MB) is due to the domain mixing mechanism, while the Params Size is consistent with RECCE (90.86MB), indicating that the storage overhead has not increased significantly.}
\subsection{Ablation Study}
\textcolor{black}{
\begin{figure*}[t]
    \centering
    \includegraphics[width=\linewidth]{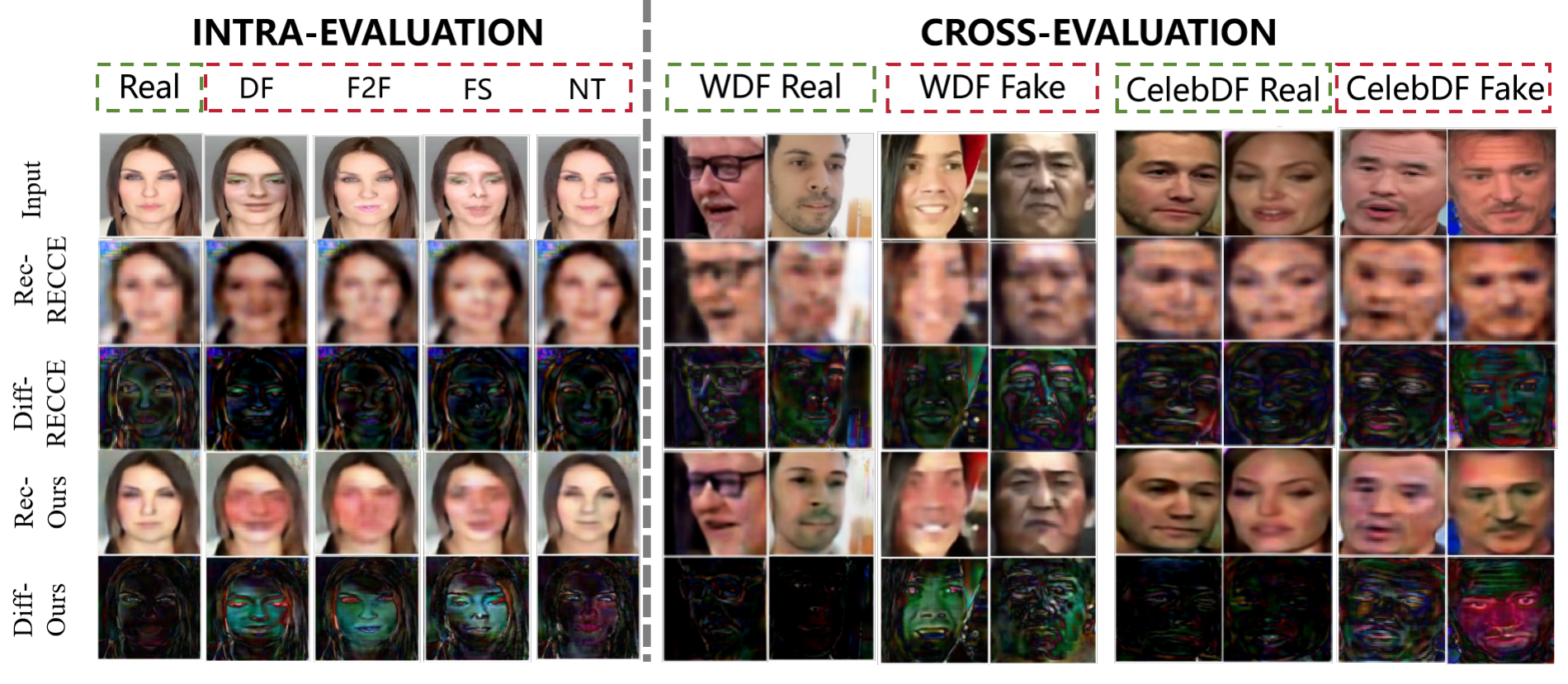}
    \caption{The representation space differences between our CDN and RECCE methods are illustrated through reconstruction and residual images on the FaceForensics++ dataset. The first row shows the original images. The second and fourth rows display the reconstructed image from the RECCE and our CDN representations, respectively, using their decoders. The third row ("Diff-RECCE") and the fifth row ("Diff-Ours") present the residual maps, which compute the pixel-level differences.
Residual maps demonstrate model performance by highlighting the distinction between forged and genuine samples. Darker areas indicate better reconstruction for genuine faces, while brighter areas signify greater divergence for forged faces, reflecting superior detection capability.}
    \label{fig:visualization}
\end{figure*}}
\paragraph{Module effects. } 
\begin{figure}[h]
    \centering
    \includegraphics[width=1\linewidth]{  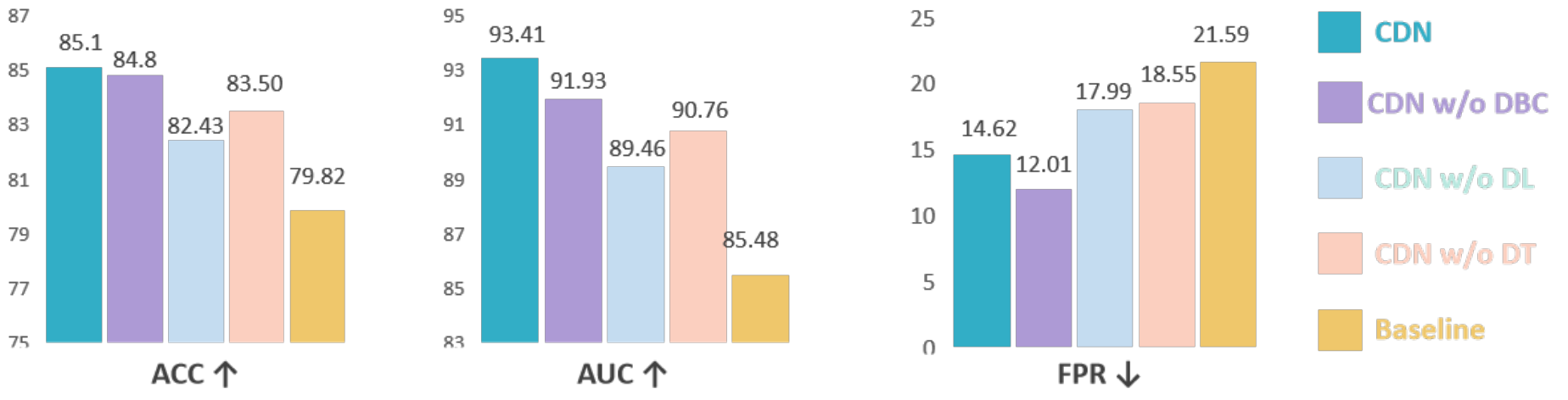}
    \caption{Ablation studies in terms of ACC (\%), AUC (\%) and FPR (\%) including Desensitization Learning (DL) and Domain Transformation (DT).}
    \label{fig:ablation}
\end{figure}

\begin{figure}[h]
    \centering
    \includegraphics[width=\linewidth]{  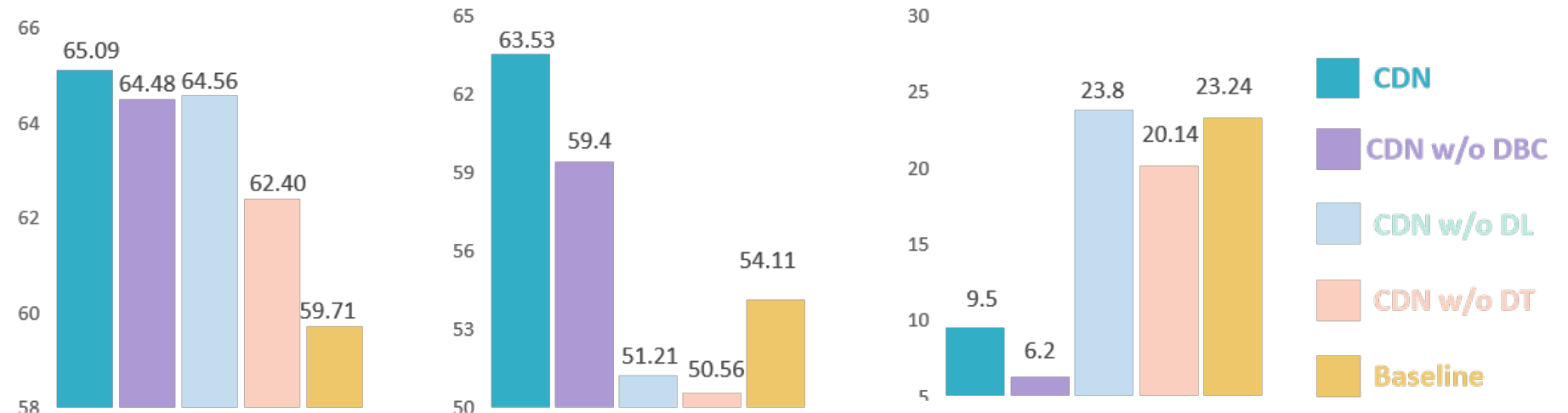}
    \caption{Ablation studies in cross-dataset setting, while testing on FF++(c23) and trained on WildDeepfake.}
    \label{fig:ablation_dbc_crossdomain}
\end{figure}
We conducted ablation experiments on WildDeepfake~\cite{WildDeepfake} dataset with two different components (i.e., Domain Transformation (DT), Desensitization Learning (DL)) removed separately to validate their contribution to the effectiveness of the proposed method under intra-dataset evaluation and cross-dataset evaluation setting. 

The intra and cross-evaluation results are given in Figure \ref{fig:ablation} and Figure \ref{fig:ablation_dbc_crossdomain}, where the baseline method (rightmost) is the CDN network without using both DT and DL components. From the figure we observe that each module is beneficial to the overall performance but it seems that Desensitization Learning (DL) is more important for the improvement of the detection accuracy compared to the DT component (see first two sub-figures), while both components are useful in reducing the False Positive Rate (see the rightmost subfigure). Indeed, the overall FPR performance will be significantly influenced if either the DT component or the DL component are removed from the whole pipeline.

In the experimental section, we conducted ablation studies on the WildDeepfake\cite{WildDeepfake} dataset to validate the impact of incorporating face-forged images on the detector's performance. The results in Table \ref{ablation_dbc} substantiate our hypothesis - Domain Boundary Constraint(DBC) module can constrain the manifold of real images by limiting the representational information of the synthesized images, thereby reducing the false negative rate (FNR). However, this approach also leads to a slight increase in the false positive rate (FPR).

Notably, even without the DBC module, our model still achieved impressive results.
\begin{table}[htbp]
    \centering
    \begin{tabular}{lcccccc}
        \toprule
        Method & ACC & AUC  &FNR &FPR\\
        \midrule
        CDN w/o DBC& 84.80 & 91.93  &20.03 &12.01\\
        CDN(Ours)& 85.21 & 93.41  &14.96 &14.62\\
        \bottomrule
    \end{tabular}
    \caption{Ablation studies of introducing fake image on intra-training in WildDeepfake\cite{WildDeepfake} ACC(\%), AUC (\%), False Negative Rate(\%) and False Positive Rate(\%)}
    \label{ablation_dbc}
\end{table}

\paragraph{Effect of Domain Boundary Constraint}
Additionally, we have integrated a  Domain Boundary Constraint (DBC) module into the architecture of our proposed CDN. Under the same experimental settings (as described in Section \ref{section_Cross_domain}) of cross-manipulation evaluation and multi-source manipulation evaluation, the experimental results shown in Tables \ref{cross_domain_dbc} and \ref{cross_manupulation_dbc} indicate that DBC can moderately improve the generalization performance of CDN, but the degree of improvement is limited. 
\begin{table}[h]
    \centering
    \small
    \setlength{\tabcolsep}{3.5pt}
    \setlength{\abovecaptionskip}{0.cm}
    \setlength{\belowcaptionskip}{0cm}
    \begin{tabular}{lccccc}
        \toprule
         Methods& GID-DF & GID-F2F & GID-FS & GID-NT  \\
        \midrule
        CDN w/o DBC&  {77.7}/{86.6} &{76.4}/{85.1}     
         &  {63.4}/{74.7}          &{66.3}/{75.4}\\
        CDN(Ours)&  {77.8}/{87.0} &{76.8}/{85.7} 
        &  {66.0}/{75.3}          &{67.6}/{76.7}\\
        \bottomrule
    \end{tabular}
    \caption{Multi-source evaluation results in terms of ACC (\%)/AUC (\%).}
    \label{cross_manupulation_dbc}
\end{table}
This suggests that the CDN model introduced in this work can achieve good generalization performance and maintain a relatively low false positive rate using only genuine image samples for training.
\begin{table}[h]
\setlength{\tabcolsep}{3pt}
    \setlength{\abovecaptionskip}{0cm}
    \setlength{\belowcaptionskip}{0cm}
    \centering
    \small
        \begin{tabular}{lcccccc}
        \toprule
        Methods  & Train & DF & F2F & FS & NT &Cross Avg.\\
        \midrule
        CDN w/o DBC	& DF&\cellcolor[gray]{0.9}99.63 &72.46	&70.78	&77.93&73.72 \\
        CDN(Ours)	& &\cellcolor[gray]{0.9}99.65 &72.38	&71.68	&79.51 &74.52\\
        \midrule
        CDN w/o DBC	& F2F&76.77	&\cellcolor[gray]{0.9}97.94	&64.71&75.92 &72.47\\
        CDN(Ours)	&   &85.86	&\cellcolor[gray]{0.9}98.93	&66.09&74.68 &75.54\\

        \midrule
        CDN w/o DBC   & FS &83.34 &66.84 &\cellcolor[gray]{0.9}98.97&60.83&70.34\\
        CDN(Ours)   &  &84.13	&66.38 &\cellcolor[gray]{0.9}99.07&61.07&70.53\\
        \midrule
        CDN w/o DBC	& NT &84.40	&81.93	&64.38	&\cellcolor[gray]{0.9}93.48&  76.90\\
        CDN(Ours)	&  &88.44	&82.72	&65.67	&\cellcolor[gray]{0.9}96.27	& 78.94\\
        \bottomrule
    \end{tabular}
    \caption{Cross-manipulation evaluation in terms of AUC (\%).}
 \label{cross_domain_dbc}
\end{table}
\paragraph{Effect of various feature layers}\label{ablation:position} 

As a desensitization learning method based on feature dimensions, our CDN can be flexibly applied to different layers during feature extraction. In our implementation, we integrate CDN into the Entry Flow of the Xception backbone, applying domain transformation across two convolutional layers. 

For notation purposes, Layer1 and Layer2 indicate that CDN is applied after the first and second convolutional layers, respectively. In this section, we conduct ablation experiments on both intra-dataset and cross-dataset evaluations using the WildDeepfake~\cite{WildDeepfake} dataset. The results, shown in Table \ref{ablation_layer}, reveal that applying feature transformation at Layer2 outperforms Layer1 in effectiveness but comes with a higher false alarm rate. Furthermore, when feature transformation is applied across the entire Entry Flow of Xception~\cite{Xception}, we achieve the highest AUC performance in both intra-dataset and cross-dataset evaluations.

\begin{table}[h]
    \centering
    \scriptsize
    \begin{tabular}{lcccccc}
        \toprule
        No. & Layer1 & Layer2  &WildDeepfake &FF++&Celeb-DF \\
        \midrule
        (a)& \checkmark &  &90.34/8.20 &48.45/3.51&56.01/41.08\\
        (b)& &\checkmark  &91.32/10.28  &59.67/4.13&65.10/34.95\\
        (c)& \checkmark&\checkmark &91.93/12.01 &61.10/22.67 &71.53/29.44\\
        \bottomrule
    \end{tabular}
    \caption{Ablation studies in terms of different layers for domain-invariant representation learning on AUC (\%) and False Positive Rate(\%)}
    \label{ablation_layer}
\end{table}

\paragraph{\textcolor{black}{Hyperparameter Sensitivity}}
\textcolor{black}{
To analyze the impact of the effect of hyperparameters $\lambda_1,\lambda_2,\lambda_3$, we evaluates different combinations of these hyperparameters and their effects on AUC and accuracy (ACC). The results are summarized in Table \ref{tab:Parameters}.
The ablation study demonstrates that the choice of hyperparameters significantly influences model performance. When $\lambda_1$ =0.1, $\lambda_2$ =0.1, $\lambda_3$=0.1, the model achieves the highest AUC (92.50) and ACC (84.82). 
We have expanded the discussion to provide insights into the sensitivity of these hyperparameters:
The hyperparameters $\lambda_1,\lambda_2,\lambda_3$ play critical roles in balancing the trade-offs between domain-invariant feature learning, domain consistency, and domain boundary constraints, respectively.
$\lambda_1$ controls domain desensitization, balancing domain invariance and feature discriminability. A low $\lambda_1$ harms generalization, while a high $\lambda_1$ risks over-suppressing domain-specific nuances. An optimal $\lambda_1$ (e.g., 0.1) enhances generalization.
$\lambda_2$ maintains domain consistency, preserving domain-specific characteristics during desensitization. A moderate $\lambda_2$ (e.g., 0.1) improves robustness to domain shifts without compromising generalization.
$\lambda_3$ enforces domain boundary constraints, preventing over-generalization. A low $\lambda_3$ increases FNR, while a high $\lambda_3$ overly constrains the feature space. An appropriate $\lambda_3$ (e.g., 0.1) mitigates over-generalization, improving performance.
Tuning $\lambda_1$, $\lambda_2$, and $\lambda_3$ ensures robust generalization and high accuracy, as demonstrated by our ablation study, offering practical insights for real-world applications.}
            \begin{table}[h]
          \centering
          \begin{tabular}{ccc|cc}
          \toprule
             $\lambda_1$ & $\lambda_2$ & $\lambda_3$  &  AUC&ACC \\
            \midrule
            0.05& 0.1&0.2&91.22&84.19\\
            0.1& 0.1&0.1&92.50&84.82\\
            0.1& 0.1&0.1&91.46&84.12\\
            0.1&0.05&0.1&91.36&84.38\\
            
            \bottomrule
          \end{tabular}
          \caption{Parameters Used in the CDN(Intra-Dataset Setting)}
          \label{tab:Parameters}
      \end{table}

\textcolor{black}{To further investigate the impact of domain transformation intensity during training, we introduced the mixing parameter $\alpha$, which controls the degree of domain transformation within a single batch. Specifically, $\alpha$ determines the proportion of samples undergoing domain transformation, thereby influencing the model's ability to learn domain-invariant features. We conducted experiments on the FS dataset for training and evaluated the model on the DF dataset. As shown in Table \ref{tab:Parameters}, $\alpha = 0.3$ yields the best performance, achieving an AUC of 84.13, accuracy (ACC) of 68.61, and a false acceptance rate (FAR) of 7.29. This optimal value balances domain transformation and feature preservation, enhancing cross-domain generalization. Lower $\alpha$ (e.g., 0.1) results in insufficient transformation (AUC: 81.01, ACC: 68.26, FAR: 22.06), while higher $\alpha$ (e.g., 0.8) overly disrupts the feature space, degrading performance (AUC: 59.12, ACC: 53.17, FAR: 13.84). These findings underscore the importance of tuning $\alpha$ for optimal domain adaptation.}

\paragraph{Visualization. }To better understand the low false positive rate behavior of the proposed CDN method, in Figure \ref{fig:visualization} we give some illustration of the results over intra and cross-dataset evaluation. It shows that the proposed method CDN can reconstruct a higher quality of images based on the learned representation, compared to RECCE~\cite{RECCE} (the 2nd row vs. 4th row). In particular, by comparing residual images in the 3rd row with those in the 5th row, we can see that our representation effectively removes the domain noise for real face images while maintaining sufficient sensitivity to the fake face images for accurate forgery detection, even to those sampled with distribution shift. 
To better illustrate the mitigation of false positive rates by CDN, we visualize the learned representation of the common reconstruction method RECCE~\cite{RECCE} and our approach using t-sne~\cite{tsne}. Our model logically projects real samples from different datasets into overlapping regions which are generally regarded as \textit{geniune}, whereas RECCE~\cite{RECCE} aggregates real samples from cross datasets into another cluster, raising the risk of false alarm.
\begin{figure}[htbp]
   \centering
   \includegraphics[width=\linewidth]{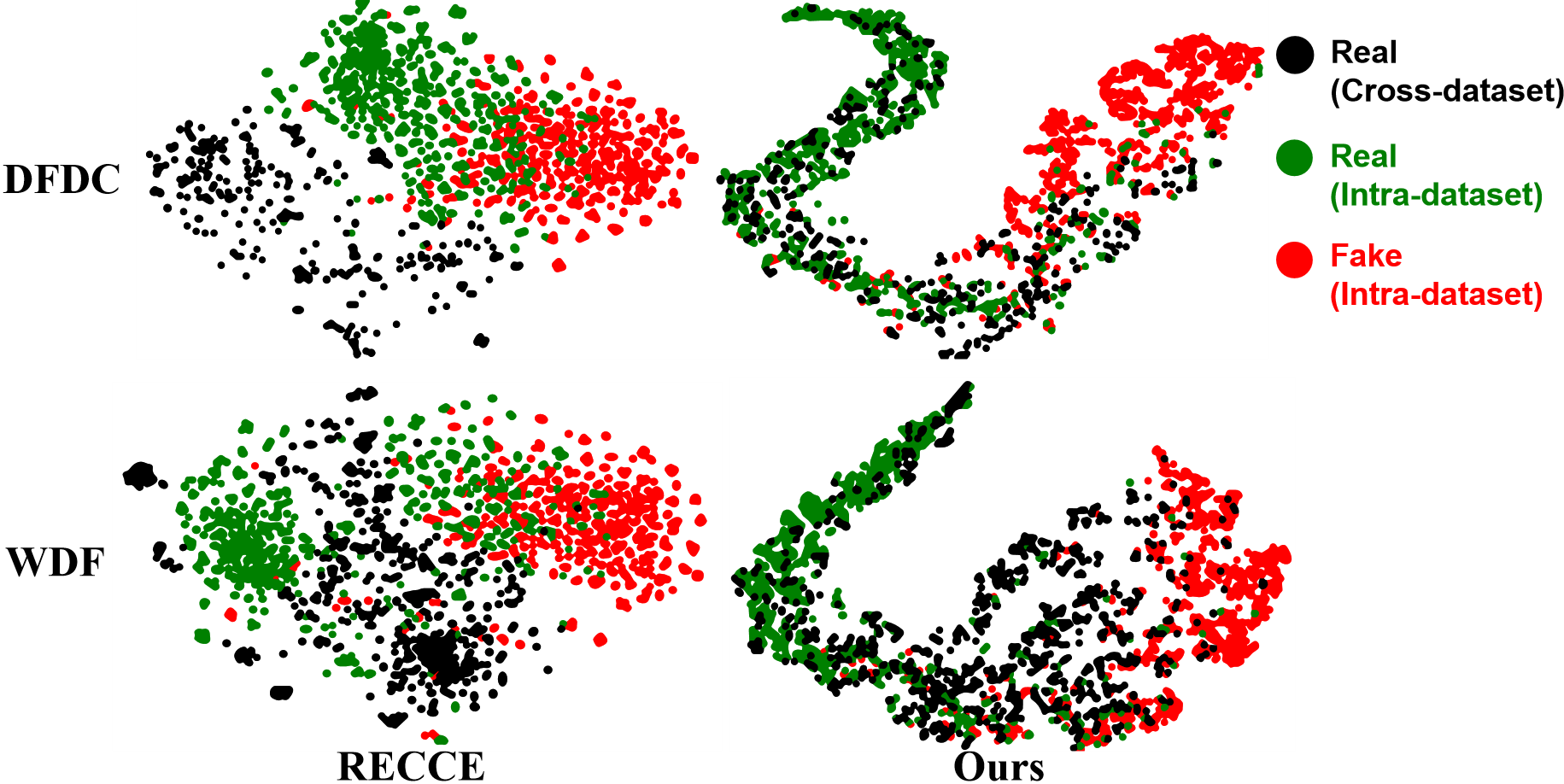}
   \caption{The t-SNE embedding visualization of the representations of RECCE and CDN(Ours). Both methods are trained on FF++(LQ) and cross-evaluation on DFDC and WildDeepfake (WDF).}
   \label{tsne}
\end{figure}

                  \begin{table}[H]
          \centering
          \begin{tabular}{c|cc|ccc}
          \toprule
             $\alpha$&Train&Test & AUC&ACC &FAR\\
            \midrule
            0.1&&&81.01&68.26&22.06\\
            0.3&FS&DF& 84.13&68.61&7.29\\
            0.8&&& 59.12&53.17&13.84\\
            
            \bottomrule
          \end{tabular}
          \caption{The degree of domain transfomation in the CDN under Cross-Manipulation Settings}
          \label{tab:degreeParameters}
      \end{table}

\textcolor{black}{ To evaluate the robustness of our method, we conducted experiments on images subjected to two types of challenging conditions: (1) For the first row of images, we applied significant compression (JPEG compression with a quality factor of 20) to simulate low-quality inputs. (2) For the second row of images, we added adversarial noise (PGD attack with $\epsilon$=8/255) before feeding them into the model.
    The results demonstrate that our method maintains strong performance even under these challenging conditions, outperforming state-of-the-art baselines. As shown in Figure \ref{fig:cam}, the activation regions are consistently focused on areas of the face where forgery traces are most evident. This indicates that our method effectively identifies and leverages key forgery-related features, even in low-quality or adversarial samples.}
      \begin{figure}[H]
    \centering
    \includegraphics[width=\linewidth]{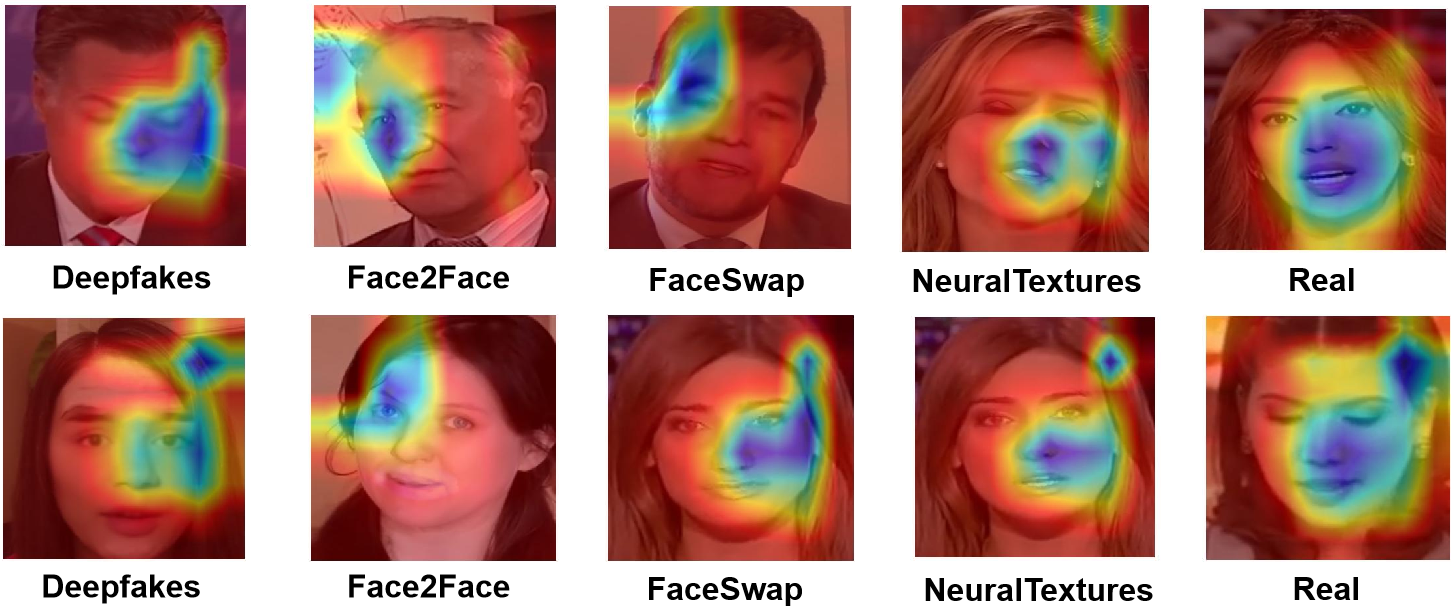}
    \caption{The GradCAM\cite{GRADCAM} visualizations of
our proposed CDN, across four forgery types on FF++(c23)}
    \label{fig:cam}
\end{figure}
\section{Conclusion}
In this work, we proposed a novel Desensitization Learning method to deal with the {\it domain shift} problem in cross-domain face forgery detection. To verify the feasibility of this idea, we implement it as a Contrastive Desensitization Network (CDN) which learns to remove the domain noise while preserving the intrinsic features of real face images. Both theoretical and experimental results demonstrate the effectiveness of the proposed method in dealing with cross-domain face forgery detection problems, although no forgery face images are used in representation learning.

Specifically, in the proposed CDN, we only use genuine faces to learn the intrinsic representation to distinguish forgery. Since we analyze the implementation approaches to the Domain Boundary Constraint (DBC) regularization to help mitigate the over-generalization issue of the distribution of real faces, it may also be excessively conservative and impact the False Alarm performance.  Therefore, in our future work, we are planning to explore alternative methods that can effectively restrict the boundaries of the real-face domain such as incorporating one-class constraints over the region occupied by real faces in the latent space.
\textcolor{black}{In the contrastive desensitization method proposed in this paper, we assume that the domain features of each domain have clear boundaries in the latent space, that is, the distance between domains is far enough, and the intrinsic features overlap sufficiently in the latent space. In reality, the limitation of the CDN method is that when the domain noise is too complex, the extracted intrinsic features and domain features may overlap too much, thus losing the discrimination performance.}
 \paragraph{Acknowledgments.}
This work is partially supported by National Science Foundation of China (6247072715).

\newpage
\bibliographystyle{cas-model2-names}
\bibliography{nn2025}

\newpage

\end{document}